\documentclass[11pt,letterpaper]{article}

\usepackage{subcaption}
\usepackage[letterpaper, total={6.5in, 9.5in}]{geometry}

\usepackage[round]{natbib}
\usepackage{xcolor}
\usepackage{graphicx}
\usepackage{amsmath}
\usepackage{amssymb}
\usepackage{amsthm}
\usepackage{subcaption}
\usepackage{enumitem}
\usepackage{algorithm}
\usepackage{algpseudocode}
\usepackage{bm}
\usepackage{comment}
\usepackage{authblk}
\usepackage{lmodern}
\usepackage[T1]{fontenc}
\usepackage{booktabs}

\newtheorem{theorem}{Theorem}
\newtheorem{lemma}{Lemma}
\newtheorem{definition}{Definition}

\newtheorem{observation}{Observation}
\newtheorem{remark}{Remark}

\usepackage[colorlinks=true]{hyperref}
\definecolor{mydarkblue}{rgb}{0,0.08,0.45}
\hypersetup{%
,urlcolor=mydarkblue
,citecolor=mydarkblue
,linkcolor=red
}

\title{Does Privacy Always Harm Fairness? Data-Dependent Trade-offs via Chernoff Information Neural Estimation}

\author{Arjun Nichani\textsuperscript{1}, Hsiang Hsu\textsuperscript{2}, Chun-Fu (Richard)  Chen\textsuperscript{2}, Haewon Jeong\textsuperscript{1} \\
\textsuperscript{1}University of California, Santa Barbara\\
\textsuperscript{2}JP Morgan Chase Global Technology Applied Research\\
\texttt{\{anichani, haewon\}@ucsb.edu}\\
\texttt{\{hsiang.hsu, richard.cf.chen\}@jpmchase.com }
}

\begin{document}

\maketitle

\begin{abstract}
    Fairness and privacy are two vital pillars of trustworthy machine learning. Despite extensive research on these individual topics, their relationship has received significantly less attention. In this paper, we utilize an information-theoretic measure \textit{Chernoff Information} to characterize the fundamental trade-off between fairness, privacy, and accuracy, as induced by the input data distribution. We first propose \textit{Chernoff Difference}, a notion of data fairness, along with its noisy variant, \textit{Noisy Chernoff Difference}, which allows us to analyze both fairness and privacy simultaneously. Through simple Gaussian examples, we show that Noisy Chernoff Difference exhibits three qualitatively distinct behaviors depending on the underlying data distribution. To extend this analysis beyond synthetic settings, we develop the Chernoff Information Neural Estimator (CINE), the first neural network–based estimator of Chernoff Information for unknown distributions. We apply CINE to analyze the Noisy Chernoff Difference on real-world datasets. Together, this work fills a critical gap in the literature by providing a principled, data-dependent characterization of the fairness–privacy interaction.
\end{abstract}

\section{Introduction}\label{sec:intro}

As machine learning (ML) systems are increasingly deployed in high-stakes settings, ensuring their trustworthiness has become critical. 
Among the many facets of trustworthy ML, fairness and privacy stand out as core concerns, yet both are frequently violated in practice. 
Unfairness has been documented across domains such as facial recognition \citep{garvie2016facial}, text-to-image generation \citep{friedrich2023fairdiffusioninstructingtexttoimage, Bianchi_2023}, and predictive tasks like recidivism \citep{barocas2016big, chouldechova2016fairpredictiondisparateimpact} and loan approvals \citep{das2021fairness}, prompting a broad literature on fair learning \citep{hardt2016equality, agarwal2018reductions, alghamdi2022beyond}. 
Privacy risks, highlighted by attacks that recover sensitive training data \citep{shokri2017membership, carlini2021extracting} and amplified by the rise of large-scale AI models \citep{Gomstyn2024}, have likewise spurred extensive research on differentially private (DP) learning \citep{Abadi_2016, papernot2018scalable, kang2020inputperturbationnewparadigm}.

\begin{figure}
    \centering
    \includegraphics[width=\linewidth]{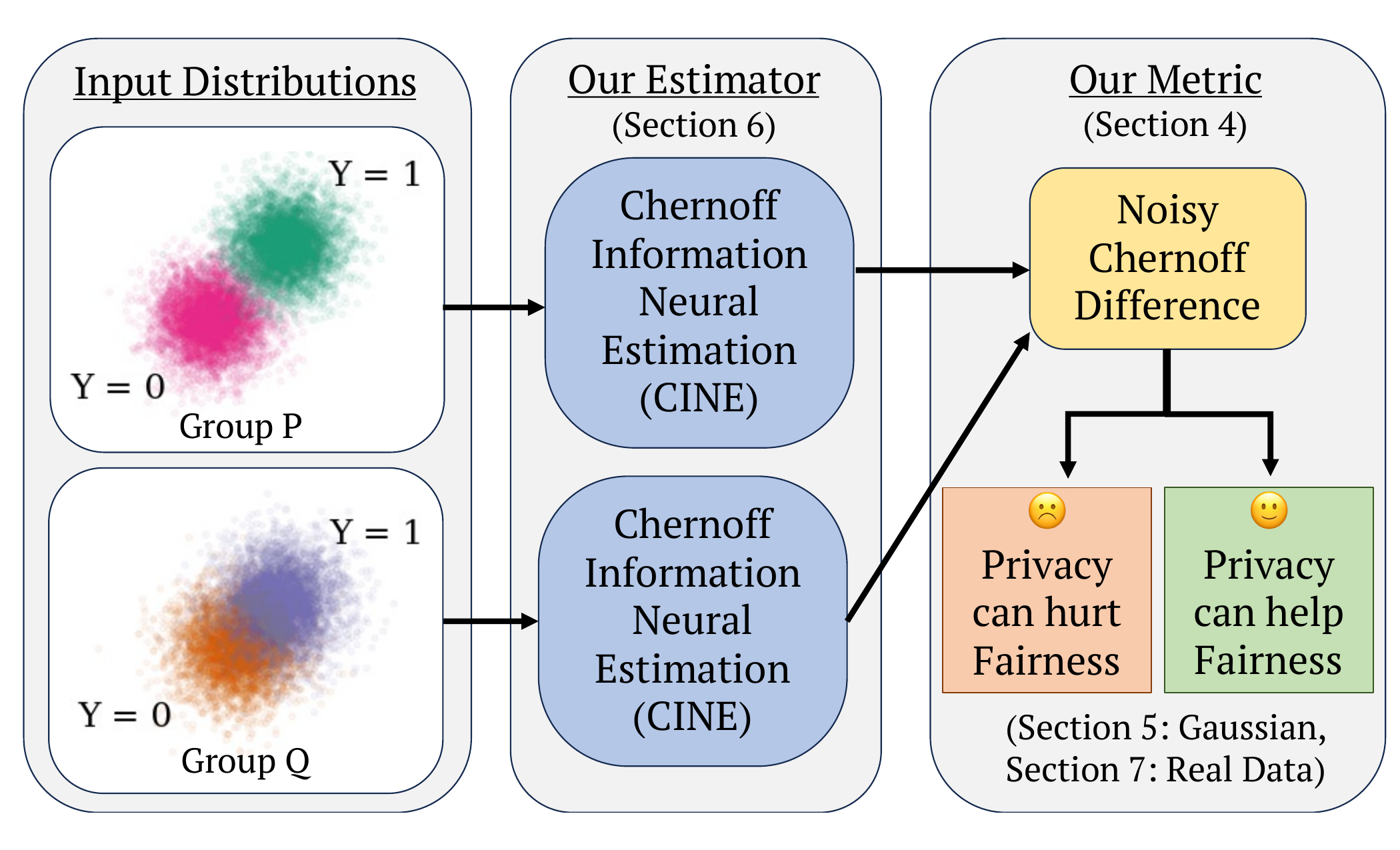}
    \caption{ \textbf{Summary of our contributions.} Starting from group-specific class-conditional distributions (left), we use CINE to estimate their Chernoff Information (middle), and compute the Chernoff Difference to diagnose whether privacy can help or hurt fairness (right).} \vspace{-3mm}
    \label{fig:placeholder}
\end{figure}

Yet while fairness and privacy have been deeply studied in isolation, their interaction remains less understood. 
Can models satisfy both simultaneously, or does enforcing one inevitably compromise the other? 
Recent work points to both incompatibility, exposing trade-offs between the two \citep{bagdasaryan2019differential, chang2021privacy, tran2021decision, sanyal2022unfair}, and potential compatibility under specific assumptions \citep{cummings2019compatibility, mangold2023differential, shahamsurvey}. In this work, we show that the answer can lie in the structure of the data itself.  We make the following contributions:

\begin{itemize}
    \item We introduce \textbf{Chernoff Difference (CD)}, a metric grounded in Chernoff Information, which serves as a notion of \textit{data fairness} by characterizing how groups differ in their classification separability. We then propose \textbf{Noisy Chernoff Difference}, an extension of CD that captures the fairness--privacy--accuracy triad by modeling how group separability evolves under privacy-preserving noise.
    \item Under Gaussian assumptions, we prove that fairness–privacy interactions fall into three distinct regimes: privacy may erode fairness, leave it unaffected, or even improve it, effectively granting \textit{``fairness for free''} when distributional conditions align.

    \item To bridge theory with practice, we develop, to our knowledge, the first algorithm to estimate Chernoff Information directly from data, leveraging advances in density ratio estimation via neural networks \citep{choi2022density}, named \textbf{Chernoff Information Neural Estimation} (CINE in Section~\ref{sec:estimation}).
    \item We apply CINE on more complex datasets, mixture of Gaussians, MNIST \citep{lecun1998gradient}, and UCI Adult \citep{adult_2}, and demonstrate that CD can capture the fairness–privacy dynamics in a wide range of settings.
    
\end{itemize}


\vspace{-2mm}
To the best of our knowledge, our findings provide the first concrete evidence that whether fairness and privacy are in conflict or in harmony is not a universal law, but a property of the underlying data distribution, 
thereby building the core foundation of a data-dependent understanding of the complex relationships between fairness, privacy, and accuracy.



\section{Related Works} \label{sec:related}

\textbf{Privacy and Fairness.}
The relationship between fairness and privacy has attracted increasing attention in recent years. \citet{bagdasaryan2019differential} show that DP disproportionately harms the accuracy of underrepresented groups, while \citet{chang2021privacy} demonstrate that minority groups are more vulnerable to membership inference attacks.  \citet{tran2021differentially} demonstrate that privacy can amplify disparities for groups closer to the decision boundary.
\citet{cummings2019compatibility} prove the impossibility of achieving \textit{exact} fairness and DP with nontrivial accuracy, but propose an algorithm that satisfies DP with approximate fairness under relaxed constraints. \citet{sanyal2022unfair} show that under long-tailed, imbalanced distributions and strict privacy, enforcing both fairness and privacy worsens accuracy, whereas \citet{mangold2023differential} establish that fairness disparities between private and non-private models decay inversely with sample size.
Several works also develop algorithms jointly ensuring fairness, privacy, and accuracy \citep{lyu2020differentiallyprivaterepresentationnlp, lowy2023stochasticoptimizationframeworkfair, ghoukasian2024differentially, jagielski2019differentially}. 
For a survey, see \citet{fioretto2022differential} and \citet{shahamsurvey}. Despite these advances, there is still a limited understanding of how the input data distribution shapes this relationship, as highlighted in a recent review paper \citep{yao2025sok}. Our work addresses this gap by providing a distribution-level characterization of how privacy reshapes group separability, offering a principled explanation of when privacy may harm or help fairness.
\paragraph{Chernoff Information.}
Chernoff Information, introduced by \citet{chernoff1952measure}, determines the optimal error exponent of Bayes-optimal classification. 
 \citet{nielsen2011, Nielsen_2013, nielsen2022revisiting} provides key properties we exploit throughout this work. 
Although information-theoretic tools have been applied to fairness and privacy \citep{ghassami2018fairnesssupervisedlearninginformation}, Chernoff Information itself has rarely been used in these domains. 
\citet{dutta2020there} show that disparities in Chernoff Information imply a trade-off between equal opportunity fairness and accuracy, while \citet{unsal2024} define Chernoff DP, where output distributions for neighboring inputs must have bounded Chernoff Information.
To our knowledge, our work is the first to connect Chernoff Information with the joint analysis of accuracy, fairness, and privacy, and to provide an estimator beyond synthetic settings.
\section{Problem Setting and Background}


 Consider a binary classification setting in which our data is defined by continuous, non-sensitive features $X$, sensitive attributes $S = \{0,1\}$, and labels $Y = \{0, 1\}$. Using these parameters, we can define our data distribution as a mixture of conditional distributions $P_0(x) = \Pr(X|S = 0, Y=0)$, $P_1(x) = \Pr(X|S = 0, Y=1)$, $Q_0(x) = \Pr(X|S = 1, Y=0)$, and $Q_1(x) = \Pr(X|S = 1, Y=1)$. From this definition, we can refer to our groups as $P$, to be the group where $S=0$ and $Q$, to be the group where $S=1$. Further, we make an equal prior assumption ($P(Y=0| S = s) = P(Y=1|S = s)$ for all $s$). We discuss a relaxation of this assumption in Section \ref{sec:real_data}.  For our model, we consider a split classifier setting where we train a different classifier for $P$ and $Q$. Under the assumption of infinite model complexity (and sufficient group information), any model will converge towards the split classifier setting \citep{wang2021split}.

 \textbf{Goal.} Our central aim is to examine how different underlying data distributions ($P_0, P_1, Q_0,$ and $Q_1$) affect the fundamental trade-off between fairness, privacy, and accuracy for a classifier.

\textbf{Privacy.} We consider differential privacy in this paper as defined below:

\begin{definition}
     A randomized learning algorithm \(\mathcal{A}\) is \((\varepsilon,\delta)\)-DP if, for every pair of neighboring datasets \(D\) and \(D'\) that differ in at most one element, and for every measurable subset \(S\) of the output space of \(\mathcal{A}\),
\begin{equation}
\Pr\!\bigl[\mathcal{A}(D) \in S\bigr]
\;\le\;
e^{\varepsilon}\,
\Pr\!\bigl[\mathcal{A}(D') \in S\bigr]
+
\delta.
\end{equation}
\end{definition}
Differentially private machine learning methods are commonly divided into three categories~\citep{jarin2022dp}: (1) input perturbation, (2) gradient/objective perturbation, and (3) output perturbation. Here, we focus on \textit{input perturbation}~\citep{fukuchi2017differentially,kang2020inputperturbationnewparadigm}, as it aligns with our goal of examining the effect of input data distribution on fairness and privacy. Empirical results further suggest that input perturbation offers advantages in regimes requiring high privacy constraints~\citep{zhao2020not}.

A standard approach to achieve $(\varepsilon, \delta)$-DP for input perturbation privacy is the Gaussian mechanism~\citep{dwork2014algorithmic}, which adds independent Gaussian noise $\sim \mathcal{N}(0, \eta^2)$ to each coordinate. Here, the noise parameter $\eta^2$ scales inversely to $\epsilon$ and $\delta$ \citep{dwork2014algorithmic}. Prior work on input perturbation attempts to determine, under specific conditions, the minimal $\eta^2$ required to achieve $(\varepsilon, \delta)-$DP, such as strong convexity of the loss function~\citep{kang2020inputperturbationnewparadigm} or a quadratic loss~\citep{fukuchi2017differentially}. However, the fundamental relationship between $\eta^2$ and $(\varepsilon, \delta)$ remains unchanged: larger $\eta^2$ corresponds to smaller $\varepsilon$ and $\delta$, indicating stronger privacy guarantees. Since our goal is to study \textit{whether} adding privacy tends to help or hurt fairness, rather than to quantify a precise privacy–fairness tradeoff for a specific mechanism, we abstract away these assumptions and adopt $\eta^2$ as our measure of privacy for the rest of the paper. For completeness, we provide an overview of the input perturbation framework and characterize the relationship between $\eta^2$ and $\varepsilon$ under different sets of assumptions in Appendix \ref{sec:fairpriv}.

\section{Chernoff Difference (CD)}

Here we introduce Chernoff Information and its relevant properties building towards Chernoff Difference and its noisy variant. For more discussion, please refer to Appendix \ref{sec:ci}. Recent work \citep{dutta2020there} has demonstrated that Chernoff Information can characterize the fundamental trade-off between fairness and accuracy. Building on this, we will leverage Chernoff Information as a tool to quantify the relationship between fairness, \textit{privacy,} and accuracy. We begin with the formal definition below.

\begin{definition}
    [Chernoff Information \citep{chernoff1952measure}] For two distributions $P_0(x)$ and $P_1(x)$ the Chernoff Information is given by:
    \begin{equation}
        C(P_0, P_1) = -\inf_{u\in(0,1)} \log(\int P_0(x)^{1-u}P_1(x)^udx).    
    \end{equation}
\end{definition}
Chernoff Information (CI) can be interpreted as a divergence that quantifies the \textit{separability} between $P_0$ and $P_1$ \citep{dutta2020there, nielsen2011}. For large values of CI, two hypotheses $P_0$ and $P_1$ are more separable, which indicates easier classification. Conversely, smaller values of CI imply less separability, and thus a more difficult classification setting. This is demonstrated by the role of CI in bounding the error exponent of the Bayes optimal classifier.

\begin{lemma}[\citep{nielsen2011}] \label{lem:bound}
    For hypotheses $P_0(x)$ under $Y=0$ and $P_1(x)$ under $Y=1$, CI bounds the error of the Bayes Optimal Classifier $H$:
    \begin{equation} \label{eq:CI_def}
        \Pr[H(X) \neq Y] \leq e^{-C(P_0, P_1)}.    
    \end{equation}
    
\end{lemma}
Asymptotically, \eqref{eq:CI_def} establishes an equivalence between Chernoff Information and the error of the Bayes-optimal classifier. Thus, CI provides a tight characterization of how well the best possible classifier performs given sufficiently large data (Appendix \ref{sec:tightness}). We next define Chernoff Difference (CD).

\begin{definition}
    [Chernoff Difference] The Chernoff Difference between group $P$ and group $Q$ is defined as follows:
    \begin{equation} \text{CD} = \left|C(P_0, P_1) - C(Q_0, Q_1)\right|.\end{equation}
\end{definition}
CD serves as a metric that compares the classification difficulty between two groups. A smaller CD indicates that the groups have similar separability between their positive and negative classes, whereas a larger CD implies a large disparity in separability. A recent work~\citep{dutta2020there} used CD to show an impossibility result: if CD is not zero, it is impossible to achieve fairness without sacrificing accuracy. While \citet{dutta2020there} draws a connection between CD and fairness, their analysis was limited to whether CD equals zero or not. In this work, we go one step further by relating the value of CD to the steepness of the fairness–accuracy curves. Below we clarify this relation, how CD operates as a notion of fairness, and how it relates to traditional notions of fairness like equal opportunity (EO).
\subsection{Chernoff Difference as a Data Fairness Measure}

Unlike model fairness metrics (e.g. EO), which quantify disparities induced by a particular trained classifier, we can think of CD as a  \textbf{data fairness} metric. The role of CD is not to evaluate a learned predictor, but to characterize the best achievable error behavior permitted by the data distributions $(P_0,P_1)$ and $(Q_0,Q_1)$\footnote{This can be extended to multi-group settings, Appendix \ref{ssec:multigroup}.}. 


This can be more formally examined by relating CD to error rates in the asymptotic regime. Formally, CI is equivalent to the error exponent of the Bayes-optimal classifier,
\begin{equation}
C(P_0,P_1) = \min\{E_{\mathrm{FP}}^P, E_{\mathrm{FN}}^P\},
\end{equation}
where $E_{\mathrm{FP}}$ and $E_{\mathrm{FN}}$ are the false positive and false negative error exponents. Consequently, CD is
\begin{equation}
\mathrm{CD}
=
\bigl|
\min\{E_{\mathrm{FP}}^P,E_{\mathrm{FN}}^P\}
-
\min\{E_{\mathrm{FP}}^Q,E_{\mathrm{FN}}^Q\}
\bigr|.
\end{equation}
Using the asymptotic relationships $E_{FP} \approx -\frac{1}{n}\log \text{FPR}$ and $E_{FN} \approx -\frac{1}{n}\log \text{FNR}$, CD corresponds to a disparity in the \textbf{dominant error rates on a log scale}:
\begin{equation}
\label{eq:log-error}
\mathrm{CD}
\approx
\frac{1}{n}\log \left|\frac{\max\{\mathrm{FPR}(P),\mathrm{FNR}(P)\}}{\max\{\mathrm{FPR}(Q),\mathrm{FNR}(Q)\}}\right|.
\end{equation}

Thus, CD captures \textit{a geometric difference} in the error rates.

\begin{remark} [Connection to existing fairness notions] CD captures \textit{a geometric difference} in the error rates, as opposed to  the arithmetic difference, commonly used in existing fairness metrics (e.g., EO). If the maximum error for both $P$ and $Q$ occurs at the FNR, then CD effectively captures the best possible FNR gap achievable by any classifier given the underlying data. Thus, CD plays a role analogous to Equal Opportunity. Both quantify a FNR gap, but while Equal Opportunity measures an additive difference, CD measures a multiplicative one. While these notions are not equivalent, they are related in meaningful ways. In the case of a fixed classifier for a single group, the relationship between the additive and multiplicative notion of fairness is monotonic. In the more general case, CD can serve as a rough lower and upper bound for the arithmetic-difference fairness metric (Theorem \ref{thm:lipschitz}) which can be tightened with knowledge of the error rates. In the following section, we empirically demonstrate that CD provides insight into the behavior of the EO fairness accuracy curve.
\end{remark}
\subsection{Noisy Chernoff Difference}

Finally, to incorporate privacy, we define a noisy variant of Chernoff Difference.

\begin{definition}[Noisy Chernoff Difference] \label{def:GNCD}
     For $\eta^2\geq 0$, we define the Noisy Chernoff Difference as:
    \begin{equation}
        \widetilde{\text{CD}}_{\eta^2} = \left|C(\widetilde{P}_0, \widetilde{P}_1) - C(\widetilde{Q}_0, \widetilde{Q}_1)\right|,
    \end{equation}
     where $\widetilde{P}_i = P_i + \mathcal{N}(0, \eta^2\mathbf{I})$ and $\widetilde{Q}_i = Q_i + \mathcal{N}(0, \eta^2\mathbf{I})$.

\end{definition}

While CD provides insight into the relationship between fairness and accuracy, $\widetilde{CD}_{\eta^2}$ provides a single value that can capture the relationship between all three: fairness, accuracy, \textit{and privacy}. More specifically, it allows us to analyze how adding noise (stronger privacy) affects both fairness and accuracy. We illustrate this relationship in the following section using Gaussian distributions.

\section{Gaussian Example: Privacy can provably help fairness}\label{sec:chernoff_info}

While we later show how to estimate CD from various distributions, including synthetic Gaussian mixtures and real datasets, we first provide clean insights into how different input distributions affect the fairness–privacy–accuracy trade-off using simple isotropic Gaussian examples. 

We define the conditional data distributions as follows:
 \begin{align}
     P_0 = \mathcal{N}(\mu_0, \sigma^2\mathbf{I}), 
     P_1 = \mathcal{N}(\mu_1, \sigma^2\mathbf{I}) \\ Q_0 = \mathcal{N}(\zeta_0, \tau^2\mathbf{I}), Q_1 = \mathcal{N}(\zeta_1, \tau^2\mathbf{I}),
 \end{align}  

This allows us to extend a result from \cite{dutta2020there} to derive a closed form expression for the Chernoff Difference.

\begin{lemma} \label{lemma} 
    When $P_0(x)\sim\mathcal{N}(\mu_0,\sigma^2\mathbf{I})$, $P_1(x)\sim\mathcal{N}(\mu_1,\sigma^2\mathbf{I})$, $Q_0(x)\sim\mathcal{N}(\zeta_0,\tau^2\mathbf{I})$, and $Q_1(x)\sim\mathcal{N}(\zeta_1,\tau^2\mathbf{I})$, the Chernoff Difference is given as the absolute difference of Bhattacharyya distances:
    \begin{equation}
        \text{CD} = \left|\frac{\|\mu_0 - \mu_1\|_2^2}{8\sigma^2} - \frac{\|\zeta_0 - \zeta_1\|_2^2}{8\tau^2}\right|.
    \end{equation}
     
\end{lemma}
The derivation of this closed-form expression is provided in Appendix \ref{sec:proofs} and discussed in \cite{nielsen2022revisiting}. This closed-form expression allows us to directly compute the CD from the definition of the distributions. To incorporate our notion of privacy, we obtain the noisy variant of CD by leveraging the normal-sum theorem \citep{lemons2002introduction}.

\begin{lemma} 
    Under the same distributions as Lemma \ref{lemma}, for $\eta^2\geq 0$, the Noisy Chernoff Difference is:
    \begin{equation}
        \widetilde{\text{CD}}_{\eta^2} = \left|\frac{\|\mu_0 - \mu_1\|_2^2}{8(\sigma^2+\eta^2)} - \frac{\|\zeta_0 - \zeta_1\|_2^2}{8(\tau^2+\eta^2)}\right|.
    \end{equation}
\end{lemma}
\vspace{-2mm}
To analyze this relationship in more depth, we examine the behavior of $\widetilde{\text{CD}}_{\eta^2}$ as we vary the privacy parameter $\eta^2$. This leads to a central result of this work.

\begin{theorem} \label{thm:central}
Suppose $P_0(x)\sim\mathcal{N}(\mu_0,\sigma^2\mathbf{I})$, $P_1(x)\sim\mathcal{N}(\mu_1,\sigma^2\mathbf{I})$, $Q_0(x)\sim\mathcal{N}(\zeta_0,\tau^2\mathbf{I})$, and $Q_1(x)\sim\mathcal{N}(\zeta_1,\tau^2\mathbf{I})$. Without loss of generality, we assume that $\|\mu_0 - \mu_1\|_2 \geq \|\zeta_0 - \zeta_1\|_2$. There are three behaviors of the Noisy Chernoff Difference ($\widetilde{\text{CD}}_{\eta^2}$): (i) $\widetilde{\text{CD}}_{\eta^2}$ has a maximum point, (ii) $\widetilde{\text{CD}}_{\eta^2}$ has a maximum point \underline{and} a reflection point (where $\widetilde{\text{CD}}_{\eta^2} = 0$), (iii) $\widetilde{\text{CD}}_{\eta^2}$ is non-increasing.\footnote{When $\|\mu_0 - \mu_1\|_2 = \|\zeta_0 - \zeta_1\|_2$, $\widetilde{\text{CD}}_{\eta^2}$ will always fall into this case.} The conditions for these three cases are given as follows:
    \begin{enumerate}
        \item[(i)] $\frac{\|\zeta_0 - \zeta_1\|^2_2}{\|\mu_0 - \mu_1\|^2_2} <\frac{\tau^2}{\sigma^2} < \frac{\|\zeta_0 - \zeta_1\|_2}{\|\mu_0 - \mu_1\|_2} < 1$,  \hfill \textbf{(Case 1)}
        \item [(ii)] $\frac{\tau^2}{\sigma^2} < \frac{\|\zeta_0 - \zeta_1\|^2_2}{\|\mu_0 - \mu_1\|^2_2} < 1$, \hfill \textbf{(Case 2)}
        \item [(iii)] Neither condition (i) or (ii) hold. \hfill \textbf{(Case 3)}
    \end{enumerate}    
\end{theorem}
We provide a proof for this theorem in Appendix \ref{sec:proofs}. 
The shape of the noisy CD curves for the three cases are shown in the top panels of Figure~\ref{fig:case123}.
These cases correspond to different scenarios where we have distinct relationships between fairness, privacy, and accuracy, which we illustrate through the examples below.

\begin{figure*}[t]
    \centering
    
    \begin{subfigure}[t]{0.32\textwidth}
        \centering
        \begin{subfigure}[t]{\textwidth}
            \centering
            \includegraphics[width=\textwidth]{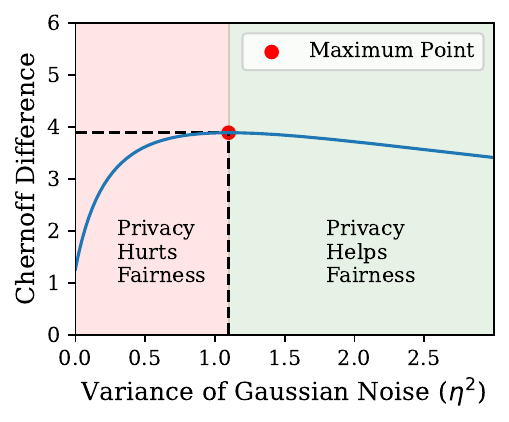}
            \caption{\footnotesize\textbf{Case 1:} Chernoff Difference}
            \label{fig:1a}
        \end{subfigure}
        
        \begin{subfigure}[t]{\textwidth}
            \centering
            \includegraphics[width=\textwidth]{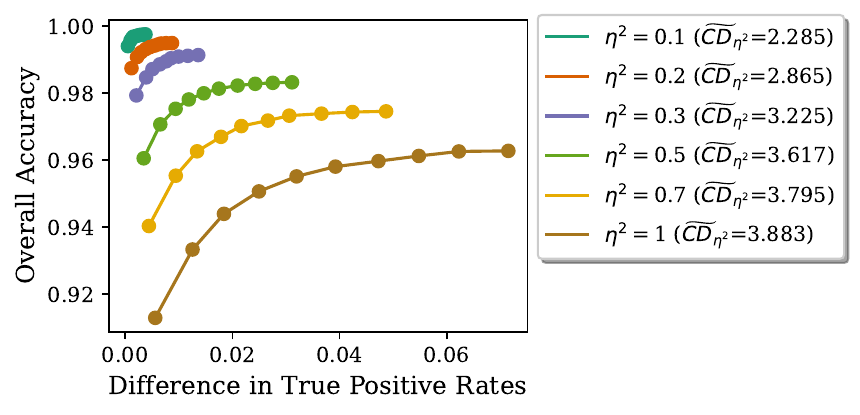}
            \caption{ \footnotesize\textbf{Case 1:} Fairness Accuracy}
            \label{fig:1b}
        \end{subfigure}
        
    \end{subfigure}
    \hfill
    \begin{subfigure}[t]{0.32\textwidth}
        \centering
        \begin{subfigure}[t]{\textwidth}
            \centering
            \includegraphics[width=\textwidth]{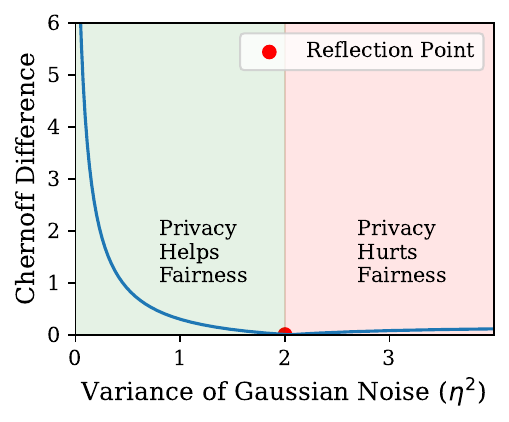}
            \caption{ \footnotesize\textbf{Case 2:} Chernoff Difference}
            \label{fig:1c}
        \end{subfigure}
        
        \begin{subfigure}[t]{\textwidth}
            \centering
            \includegraphics[width=\textwidth]{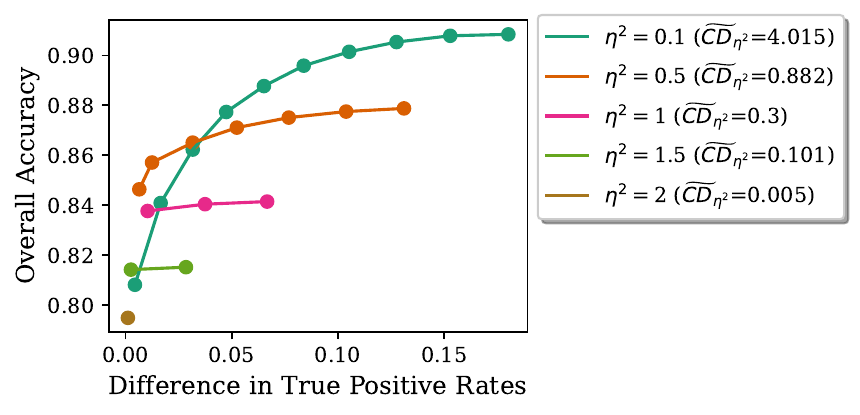}
            \caption{ \footnotesize\textbf{Case 2:} Fairness Accuracy}
            \label{fig:1d}
        \end{subfigure}
        
    \end{subfigure}
    \hfill
    \begin{subfigure}[t]{0.32\textwidth}
        \centering
        \begin{subfigure}[t]{\textwidth}
            \centering
            \includegraphics[width=\textwidth]{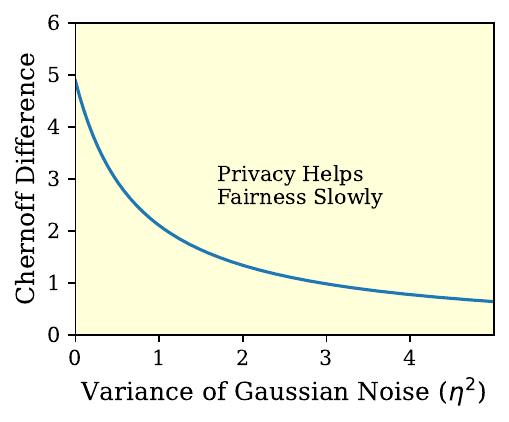}
            \caption{ \footnotesize\textbf{Case 3:} Chernoff Difference}
            \label{fig:1e}
        \end{subfigure}
        
        \begin{subfigure}[t]{\textwidth}
            \centering
            \includegraphics[width=\textwidth]{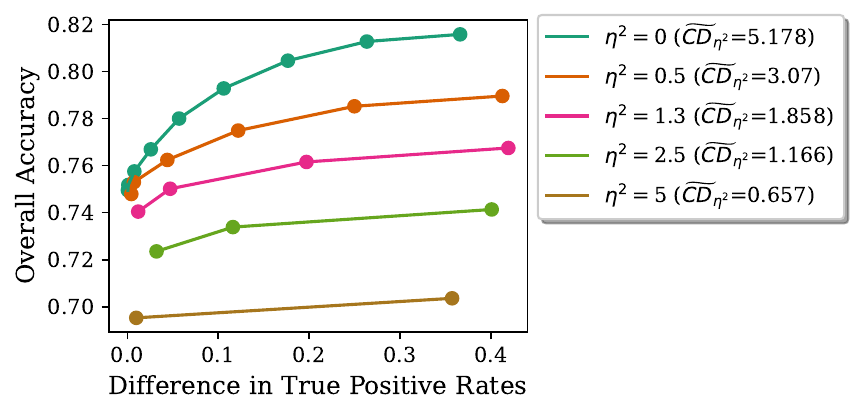}
            \caption{ \footnotesize\textbf{Case 3:} Fairness Accuracy}
            \label{fig:1f}
        \end{subfigure}
        
    \end{subfigure}
    
    \caption{\footnotesize \textbf{Illustration of three distinct cases of Gaussian given in Theorem~\ref{thm:central}. (Case 1)} (a) $\widetilde{\text{CD}}_{\eta^2}$ increases until the maximum point, which represents the worst fairness-accuracy trade-off we encounter by adding noise. (b) The slopes of the fairness-accuracy plots increase as we add noise until the $\widetilde{\text{CD}}_{\eta^2}$ reaches the maximum. 
    Parameters used here are: $\mu_0 = 0, \mu_1=16.5,\sigma = 2.43$ and $\zeta_0 = 0.5, \zeta_1 =3.8, \tau = 0.55$. 
    \textbf{(Case 2)} (c) $\widetilde{\text{CD}}_{\eta^2}$ decays until it reaches 0. It then reflects back and starts to increase, albeit with a near-zero slope. The maximum occurs at a large $\eta^2$ (6.86) value and thus is not plotted (see Appendix \ref{ssec:case2_moredetail}). (d) The steepness of the fairness-accuracy plots decreases as we add noise and CD decreases. The intersecting lines show that in some cases, we can achieve better fairness for the same accuracy when we add noise. 
    Parameters used here are: $\mu_0 = -4.2, \mu_1=1.3,\sigma = 3$ and $\zeta_0 = 0.3, \zeta_1 =2.7, \tau = 0.25$. 
    \textbf{(Case 3)} (e) $\widetilde{\text{CD}}_{\eta^2}$ decays steadily over the entire positive $\eta^2$ regime. (f) Fairness-accuracy curves flatten as $\widetilde{\text{CD}}_{\eta^2}$ decreases with the increasing noise, but the rate of decay is too slow to produce a clear crossover point, unlike in Case 2. 
    Parameters used here are: $\mu_0 = -4.2, \mu_1=1.3,\sigma = 0.85$ and $\zeta_0 = 0.6, \zeta_1 =1.6, \tau = 0.6$.
    } \vspace{-3mm}
    \label{fig:case123}
\end{figure*}

\textbf{Three Cases of Gaussian Distributions.} 
We now construct 1D Gaussian distributions that satisfy the conditions for each case in Theorem~\ref{thm:central}, and examine how different behaviors in noisy CD curves translate to different fairness-privacy relationships.  We generate 100,000 samples from each distribution, and then train Na\"{i}ve Bayes classifiers with different priors to obtain fairness-accuracy curves. See Appendix \ref{ssec:gauss_det} for more experimental details.

\textbf{Case 1: Privacy Can Hurt Fairness.} In Case 1, $\widetilde{\text{CD}}_{\eta^2}$ grows as we increase Gaussian noise until it reaches a maximum at around $\eta^2 = 1$. As $\widetilde{\text{CD}}_{\eta^2}$ increases, we observe the slopes of the fairness-accuracy curves becoming steeper (Figure \ref{fig:1b}), indicating that adding privacy worsens the fairness-accuracy trade-off. We provide log fairness-accuracy plots in Appendix \ref{appx:figures}, where this change in slope is more clearly visible. 
After the maximum point, additional noise leads to a gradual decay in $\widetilde{\text{CD}}_{\eta^2}$; however, the decay is very slow and its effect on the fairness-accuracy curves are nearly imperceptible (see Appendix~\ref{ssec:case1ext}).  

To provide intuition for Case 1, note that the group $Q$ has closer group means than the group $P$. 
At the same time, the variance of group $Q$'s distribution must fall within a specific range: it must be large enough (compared to $P$) so that $Q$ would be the unprivileged group (i.e., less separable), 
but also small enough that modest noise degrades $Q$'s separability more than $P$'s. In this regime, because $Q$ is already less separable, the gap in Chernoff Information between the groups increases as small noise is added. Once the noise becomes sufficiently large, it begins to affect $P$'s separability slightly more, marking a transition where the Chernoff Difference starts to decrease.

\textbf{Case 2 and Case 3: Privacy Can Help Fairness.} In Case 2, the noisy CD decays very rapidly to 0 as we increase $\eta^2$ (Figure \ref{fig:1c}). As $\widetilde{\text{CD}}_{\eta^2}$ decreases, the fairness-accuracy curves clearly flatten (Figure \ref{fig:1d}). In fact, the curves flatten so quickly that the noisy curves begin to overlap with the clean fairness-accuracy curves. This shows that by adding privacy, we can achieve better fairness for the same accuracy---\textit{privacy gives free fairness!} 

To explain the mechanism behind this ``free fairness,'' note that the class means of group $Q$ are still closer than those of group $P$'s, but the variance of group $Q$ is much smaller. As a result, $Q$  becomes the privileged group (i.e., better separability). As we add noise, the separability of $P$ and $Q$ both decreases, but $Q$ much faster than $P$, due to its smaller mean separation and variance. Because the privileged group loses separability faster, the Chernoff Difference decreases, effectively reducing unfairness.
This trend continues until the reflection point is reached where $\widetilde{\text{CD}}_{\eta^2} = 0$. Beyond this point, $\widetilde{\text{CD}}_{\eta^2}$ starts to increase until the maximum and then decrease again, but the slope remains near zero after the reflection point (see Appendix~\ref{ssec:case2_moredetail}).


 In Case 3, when neither condition (i) nor (ii) in Theorem \ref{thm:central} holds, $\widetilde{\text{CD}}_{\eta^2}$ smoothly decays as we increase $\eta^2$ (Figure \ref{fig:1e}). Again, this is reflected in the flattening trend of the fairness-accuracy curve in Figure~\ref{fig:1f}. However, unlike Case 2, we do not observe any crossing of the curves, as the decay of CD is slower. Improvements in fairness cannot keep pace with the accuracy degradation induced by privacy.
 Here, group $Q$ exhibits a larger variance in addition to its smaller mean separation, leading to significantly less separability compared to group $P$. As noise is added, the separability of $P$ is reduced at a rate faster than $Q$, which closes the gap in Chernoff Information. However, the separability of $Q$ is initially too small for CD to ever reach zero.

\section{Chernoff Information Neural Estimation (CINE)}\label{sec:estimation}

Beyond isotropic Gaussian distributions, closed form solutions for Chernoff Information (CI) are only available for single-parameter members of the exponential family \citep{Nielsen_2013}. On real-world data, computing CI requires density estimation, a long-standing research challenge in ML and statistics \citep{bishop2006pattern, vapnik2006estimation}. 
Here, we introduce Chernoff Information Neural Estimation (CINE), a novel algorithm that estimates CI from arbitrary distributions by 
recasting it as a density ratio estimation (DRE) problem. To the best of our knowledge, \textit{this is the first method for estimating CI directly from real-world datasets.} In order to develop our CI estimation algorithm, we first make the following crucial observation:
\begin{observation}
    Chernoff Information can be written as an optimization of an expectation 
    \begin{equation} \label{eq:expci}
    C(P_0, P_1) = -\inf_{u \in (0,1)} \log\mathbb{E}_{x\sim P_0}\left[ \left(\frac{P_1(x)}{P_0(x)} \right)^u \right].
\end{equation}

\end{observation}
 This yields \textit{a density ratio} within the expectation. Accurately estimating the density ratio $\frac{P_1(x)}{P_0(x)}$ is easier than directly estimating $P_1(x)$ and $P_0(x)$ \citep{sugiyama2012density}, as shown theoretically in \citet{kanamori2010theoretical}. We leverage this insight to develop our algorithm.

\begin{algorithm}[t]
\caption{\textbf{CINE via DRE-$\infty$ \citep{choi2022density}}} 
\label{alg:CINE} 
\begin{algorithmic}[1]
    \State \textbf{Input:} Distributions $\boldsymbol{P_0}, \boldsymbol{P_1}$
    \State $s_\theta \gets \text{Density Ratio Estimation}(P_1, P_0)$ \Comment{Compute score function}
    \State $\{\log r_i\}_{i=1}^n \gets \int_1^0s_\theta(x_i,t)[t]dt$ \Comment{Using $x_i \sim P_0$}
    \State $f(u) \gets \log\left(\frac{1}{n} \sum_{i=1}^n \exp(u \cdot \log r_i)\right)$ \Comment{Monte Carlo}
    \State $CI \gets -\inf_{u \in (0,1)} f(u)$ \Comment{Convex Optimization}
    
    \State \textbf{Return:} $CI$
\end{algorithmic}
\label{alg}
\end{algorithm}

\begin{figure*}[t]
  \centering
  \begin{subfigure}{0.305\linewidth}
    \centering
\includegraphics[width=\linewidth]{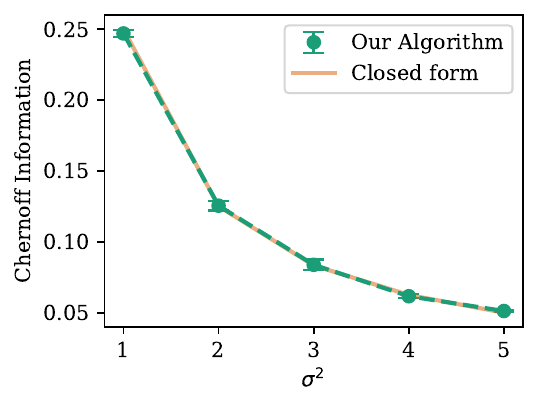}
    \caption{CINE vs Closed form CI (2-D) }
    \label{fig:gaussa}
  \end{subfigure}
  \begin{subfigure}{0.305\linewidth}
    \centering    \includegraphics[width=\linewidth]{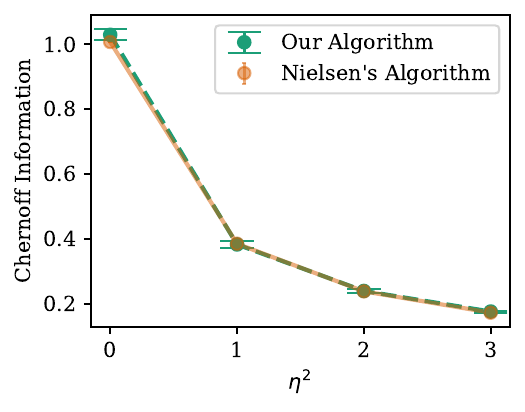}
    \caption{CINE vs Nielsen's Estimator (5-D)}
    \label{fig:gaussb}
  \end{subfigure}
  \begin{subfigure}{0.305\linewidth}
    \centering    \includegraphics[width=\linewidth]{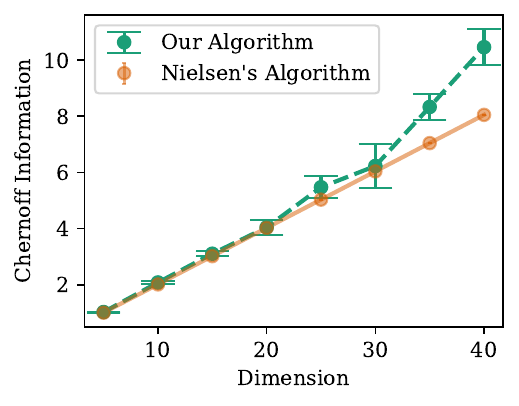}
    \caption{CINE vs Nielsen's Estimator (n-D)}
    \label{fig:gaussc}
  \end{subfigure}
  \caption{{\footnotesize 
   \textbf{Empirical evaluation of CINE.} (a) In 2D, using distributions $\mathcal{N}(\textbf{0},  \sigma^2\mathbf{I})$ and $\mathcal{N}(\textbf{1},  \sigma^2\mathbf{I})$ Chernoff Information estimates remain accurate for different values of $\sigma^2$. (b) In 5D, using distributions $\mathcal{N}(\textbf{0}, \frac{1}{2}\mathbf{I})$ and $\mathcal{N}(\textbf{1},\textbf{I})$ Chernoff Information estimation remains accurate, even as noise $(\eta^2)$ is added. (c) Using distributions $\mathcal{N}(\textbf{0}, \frac{1}{2}\mathbf{I})$ and $\mathcal{N}(\textbf{1},\textbf{I})$}, estimations are accurate in lower dimensions before degrading.
  } 
\end{figure*}

\paragraph{Background on neural density ratio estimation.} To estimate the density ratio, we adopt a state-of-the-art method known as the  \textit{telescoping approach}~\citep{choi2022density, rhodes2020telescoping}, which expresses the ratio $r(x) = \frac{P_1(x)}{P_0(x)}$ as a product of ratios between intermediate bridge distributions $p_\lambda(x)$ \footnote{While we adopt a current state-of-the-art method for DRE, Algorithm~\ref{alg:CINE} can plug in any DRE, and CINE can benefit from future developments in this area.}. These bridges, indexed by $\lambda$, are created via an interpolation scheme between the original distributions (e.g., $p_\lambda(x) = \lambda P_0(x) + \sqrt{1 - \lambda^2} P_1(x)$). This approach improves estimation accuracy by decomposing the overall ratio into smaller steps between distributions that are closer to each other and thus easier to estimate.
\citet{choi2022density} consider the case where the number of bridges approaches infinity, allowing the log-density ratio to be expressed as an integral: $\log r(x) = \int_1^0\frac{\partial}{\partial \lambda}\log p_\lambda(x)\partial\lambda$.  Observing the similarity between this telescoping process and diffusion models~\citep{nichol2021improved, song2021maximum}, they propose \textit{training a neural network to recover the time score function} $\frac{\partial}{\partial \lambda}p_\lambda(x)$, while the integral is computed using standard numerical methods. We refer the reader to the original work~\citep{choi2022density} for full details of the estimator and Appendix \ref{ssec:estexpdet} for details regarding our implementation.

Once the neural density ratio estimator is trained, we compute the expectation 
$\mathbb{E}_{x\sim P_0}[(\frac{P_1(x)}{P_0(x)})^u]$ using a Monte Carlo method, and denote it as a function $g(u)$. Finally, the last step of estimating Chernoff Information is solving $-\inf_{u \in (0,1)}g(u)$ in \eqref{eq:expci}. This optimization can be performed easily and efficiently due to the convex nature of $g$. Convexity follows from a result by \citet{nielsen2022revisiting}:
\begin{lemma} [Section 2.1 \citep{nielsen2022revisiting}] \label{lemma:convex}
The skewed Bhattacharyya distance, $\log(\int P_0(x)^{1-u}P_1(x)^udx)$ is convex with respect to $u$.
\end{lemma}


We provide a full overview of CINE in Algorithm \ref{alg}.
The key insight of CINE is that by reformulating the definition of CI, we can leverage both state-of-the-art methods for density ratio estimation~\citep{choi2022density} and recent theoretical advances on CI~\citep{nielsen2022revisiting}. Finally, we prove the consistency of the proposed CINE algorithm (Appendix \ref{sec:consist}):

\begin{theorem}
    Let $P_0$ and $P_1$ be distributions with common support and bounded density ratio. Then the Chernoff Information estimator, constructed using a consistent density ratio estimator, is itself consistent.
\end{theorem}

\subsection{Empirical Evaluation of CINE}
Next, we validate CINE by comparing its results with both the closed-form solutions of CI and the algorithm proposed by \citet{nielsen2022revisiting}. This evaluation demonstrates the accuracy of our method and its scalability with input dimension. In all experiments, we used 10,000 samples per class to train CINE. More details about the setting are in Appendix \ref{ssec:gaussexpdet}. 

We first test our algorithm on Gaussian distributions that have a closed form solution for Chernoff Information (see Lemma \ref{lemma}). We construct 2D Gaussian distributions with the equal variance: $\mathcal{N}(\textbf{0},  \sigma^2\mathbf{I})$ and $\mathcal{N}(\textbf{1}, \sigma^2\mathbf{I})$. As observed in Figure \ref{fig:gaussa}, our estimator closely matches the closed form solution across a wide range of variances ($\sigma^2$).

Next, to evaluate on distributions with no closed form solution, we compare our estimator to the multidimensional Gaussian estimator from \citet{nielsen2022revisiting}. This estimator relies on the fact that the CI between high dimensional Gaussians, can be written as a skewed KL-divergence where the optimal skewing parameter can be found with a binary search algorithm\footnote{This relies on knowledge that the distributions are Gaussian.}. For experiments, we construct 5D Gaussian distributions $N(\mathbf{0}, (\frac{1}{2} +\eta^2)\mathbf{I})$ and $\mathcal{N}(\mathbf{1}, \eta^2\mathbf{I})$. We then sweep across various noise values $\eta^2$ and show that our algorithm closely matches Nielsen's estimator (Figure \ref{fig:gaussb}).

Finally, in Figure~\ref{fig:gaussc}, we examine how the performance of CINE scales with input dimension. We construct Gaussian distributions, $N(\mathbf{0}, \frac{1}{2}\mathbf{I})$ and $\mathcal{N}(\mathbf{1}, \mathbf{I})$, with dimensions ranging from 5 to 40. With 10,000 samples per class, our estimator is nearly indistinguishable from Nielsen’s up to 20 dimensions. CINE begins to show modest variance and small estimation errors between 20 and 30 dimensions, and beyond 30 dimensions the estimator starts to clearly overestimate CI compared to Nielsen's. 
This higher dimensional degradation is a common problem when estimating information-theoretic measures from samples \citep{belghazi2018mine}, and we discuss these limitations in Section \ref{sec:conclusion}. 

\section{Fairness-Privacy Dynamic via CINE} \label{sec:real_data}
To extend analysis beyond the simple Gaussian setting, we apply CINE to  more complex distributions: mixture of Gaussians, and real-world datasets (Adult and MNIST).

\textbf{Mixture of Gaussians.} We extend beyond the isotropic Gaussian setting by utilizing a more complex mixture of Gaussian setting. For each conditional distribution, we utilize two 2D Gaussian distributions to build the mixture. We consider 2 mixtures, which we name Mixture 1 and Mixture 2 and repeat the experiments from Section \ref{sec:chernoff_info} (for more details on these mixtures and experiments, please refer to Appendix \ref{sec:exp_det}). Using Mixture 1, we demonstrate that the CD spikes as we increase the variance of the noise (Figure \ref{fig:3a}), similar to the trend we observed in Case 1 with isotropic Gaussians. We also observe a steepening in the fairness accuracy curve, as predicted by increasing CD (Figure \ref{fig:3b}). Using Mixture 2, we show that CD can quickly decay towards 0 (Figure \ref{fig:3c}), similar to Case 2 in the isotropic Gaussian setting. Further, we see that the fairness accuracy curves flatten as noise increases and we can obtain free fairness (Figure \ref{fig:3d}).

\begin{figure*}[t]
    \centering
    
    \begin{subfigure}[t]{0.24\textwidth}
        \centering
        \begin{subfigure}[t]{\textwidth}
            \centering
            \includegraphics[width=\textwidth]{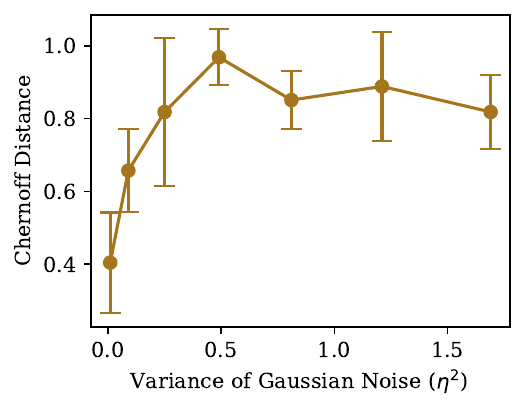}
            \caption{\footnotesize Mixture 1: CD}
        \label{fig:3a}
        \end{subfigure}
        
        \begin{subfigure}[t]{\textwidth}
            \centering
            \includegraphics[width=\textwidth]{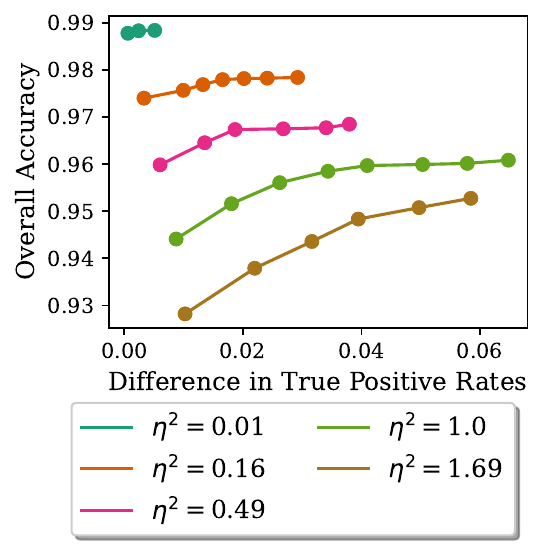}
            \caption{\footnotesize Mixture 1: Fairness-Acc}
        \label{fig:3b}
        \end{subfigure}
        
    \end{subfigure}
    \hfill
    \begin{subfigure}[t]{0.24\textwidth}
        \centering
        \begin{subfigure}[t]{\textwidth}
            \centering
            \includegraphics[width=\textwidth]{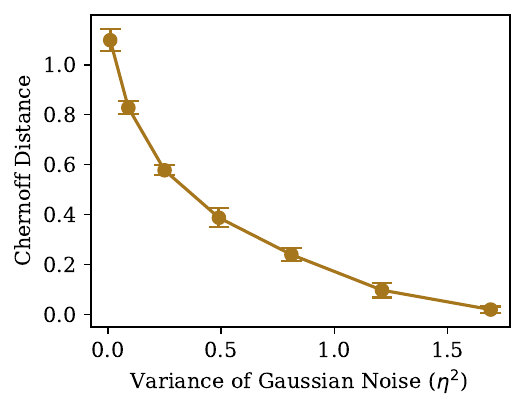}
            \caption{\footnotesize Mixture 2: CD}
        \label{fig:3c}
        \end{subfigure}
        
        \begin{subfigure}[t]{\textwidth}
            \centering
            \includegraphics[width=\textwidth]{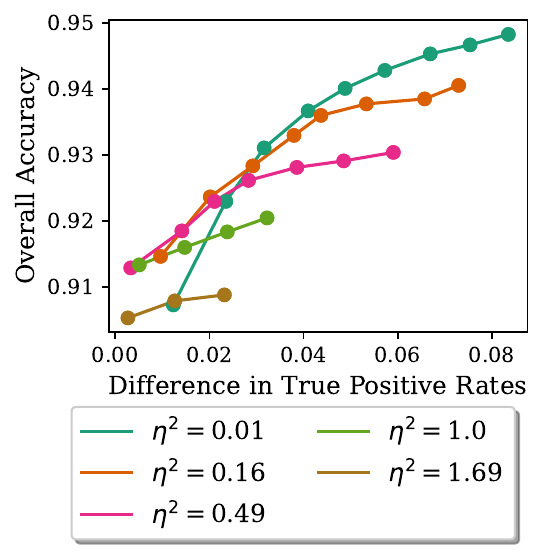}
            \caption{\footnotesize Mixture 2: Fairness-Acc}
        \label{fig:3d}
        \end{subfigure}
        
    \end{subfigure}
    \hfill
    \begin{subfigure}[t]{0.24\textwidth}
        \centering
        \begin{subfigure}[t]{\textwidth}
            \centering
            \includegraphics[width=\textwidth]{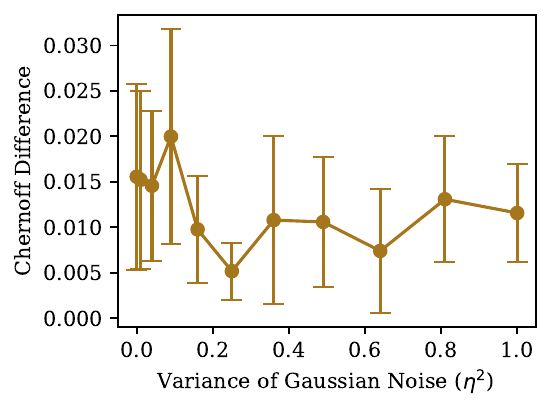}
            \caption{\footnotesize Adult: CD}
        \label{fig:3e}
        \end{subfigure}
        
        \begin{subfigure}[t]{\textwidth}
            \centering
            \includegraphics[width=\textwidth]{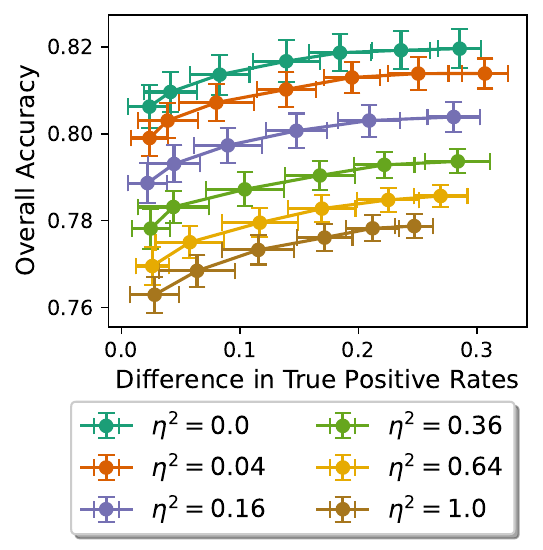}
            \caption{\footnotesize Adult: Fairness-Acc}
        \label{fig:3f}
        \end{subfigure}
        
    \end{subfigure}
    \begin{subfigure}[t]{0.24\textwidth}
        \centering
        \begin{subfigure}[t]{\textwidth}
            \centering
            \includegraphics[width=\textwidth]{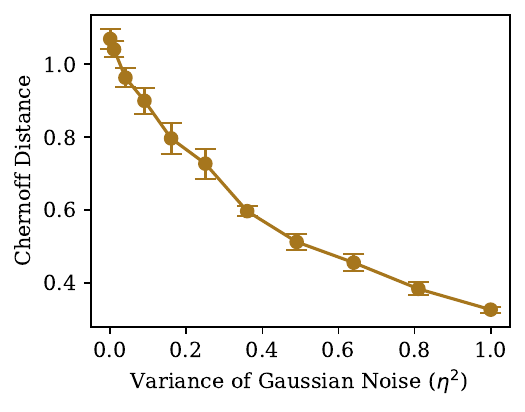}
            \caption{\footnotesize MNIST: CD}
        \label{fig:3g}
        \end{subfigure}
        
        \begin{subfigure}[t]{\textwidth}
            \centering
            \includegraphics[width=\textwidth]{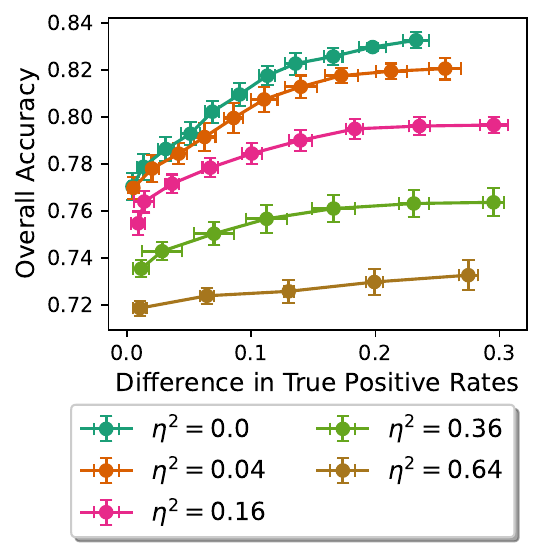}
            \caption{\footnotesize MNIST: Fairness-Acc}
        \label{fig:3h}
        \end{subfigure}
        
    \end{subfigure}

    \caption{\footnotesize \textbf{(Mixture 1)} (a) $\widetilde{\text{CD}}_{\eta^2}$ spikes as $\eta^2$ increases. (b) Fairness-Accuracy curve becomes more steep resembling fairness-accuracy curves from Case 1 from isotropic Gaussians. \textbf{(Mixture 2)} (c) $\widetilde{\text{CD}}_{\eta^2}$ decreasing to 0.  (d) Fairness-Accuracy curves flattening  resembling fairness-accuracy curves from Case 2 from isotropic Gaussians. \textbf{(Adult)} (e) $\widetilde{\text{CD}}_{\eta^2}$ Remains Flat for Adult dataset. (f) Fairness-Accuracy slopes remain stable for Adult dataset. \textbf{(MNIST)} (g) $\widetilde{\text{CD}}_{\eta^2}$ decreases for experiments on MNIST representations. (h) Fairness-Accuracy curves flatten for MNIST representations.} \vspace{-3mm}
\end{figure*}
\textbf{Adult.}
We next apply CINE on the Adult dataset \citep{adult_2}. We consider the standard income classification task and use sex as our protected attribute. We selected all the continuous features from the dataset, giving us 6 dimensions. For full details refer to Appendix \ref{ssec:dataexpdet}.

First, as the variance of the additive Gaussian noise increases, the CD remains relatively flat (Figure \ref{fig:3e}). We observe that there may be slight trends in CD, however, they appear to be too small to make meaningful differences in the fairness, privacy dynamic. This is reflected by the fairness accuracy curves (Figure \ref{fig:3f}). We train split logistic regression classifiers and perturb the class weights for the disadvantaged group to create our fairness-accuracy curves. For more details, please refer to Appendix \ref{ssec:dataexpdet}. As observed in Section $\ref{sec:chernoff_info}$, the CD acts as a proxy for the slope of the fairness accuracy curves. Here, as the input noise varies, the slopes of the fairness accuracy curves remain relatively stable. Further, as the CD is consistently low, we observe very flat fairness accuracy curves. This indicates that privacy does not have a strong effect on the relationship between fairness and accuracy for these features.

\textbf{MNIST.} Next, we apply CINE to an image classification task on MNIST. We construct a synthetic fair classification task on the MNIST dataset by assigning specific digits to represent each subgroup in our experiment. While MNIST does not involve human subjects, this setup allows us to study fairness–accuracy trade-offs in a controlled setting by synthetically defining majority and minority groups. Specifically, we choose the digit $3$ for $P_0$, the digit $4$ for $P_1$, the digit $7$ for $Q_0$, and the digit $9$ for $Q_1$. 

Using these subgroups, we follow a two-stage pipeline. We first perform dimensionality reduction using a learned autoencoder, followed by logistic regression on the resulting low-dimensional image representations. Although this is a two-step procedure, it is effectively equivalent to neural network classification, where the low-dimensional representation corresponds to the final-layer activation. We adopt this pipeline as estimating Chernoff Information directly on high-dimensional data is challenging. In contrast, applying CINE to the final-layer activations provides a meaningful proxy for the intrinsic separability of the input distribution.

For CD, we observe a steady decay as we increase the input perturbation noise (Figure \ref{fig:3g}). This trend is reflected in the fairness-accuracy curves, which exhibit a clear flattening effect (Figure \ref{fig:3h}). In particular, the first and second curves overlap within the standard deviation range, potentially indicating the \textit{free fairness} regime. These results show that CD  effectively captures the fairness-privacy-accuracy trade-off even on more complex data.

\vspace{-2mm}
\section{Conclusion and Future Work}\label{sec:conclusion}
This work proposes a new information-theoretic fairness measure, called \textit{Chernoff Difference},
and shows that it provides a principled framework for analyzing the fundamental fairness–privacy–accuracy trade-off inherent in the input data distribution. While this paper focuses on establishing the core concepts and tools, it naturally opens up several directions for future work. In particular, although we prove consistency of the proposed CINE algorithm and demonstrate strong empirical performance in moderate dimensions ($\sim$30), a deeper sample complexity analysis and further improvements to estimator performance in higher dimensions remain important directions for future study. 
Moreover, our fairness-privacy analysis currently focuses on one notion of fairness (a proxy for equal opportunity) due to its natural connection to Chernoff Information. Studying the  data-dependent relationship between privacy and fairness for  other fairness metrics (e.g., statistical parity, multicalibration) would require different theoretical tools and would be an exciting area of exploration. 
Finally, we believe that the proposed framework can inform fair and private mechanism design,
e.g., Chernoff Difference analysis may be used as a preliminary diagnostic to assess whether privacy mechanisms are likely to adversely affect fairness for a given data distribution, helping practitioners decide when additional fairness constraints are necessary. 
%

\paragraph{Acknowledgments} This work was supported by the National Science Foundation (NSF) under grant number 2341055.

\paragraph{Disclaimer.}
This paper was prepared for informational purposes by the Global Technology Applied Research center of JPMorgan Chase \& Co. This paper is not a product of the Research Department of JPMorgan Chase \& Co. or its
affiliates. Neither JPMorgan Chase \& Co. nor any of its affiliates makes any explicit or implied representation
or warranty and none of them accept any liability in connection with this paper, including, without limitation,
with respect to the completeness, accuracy, or reliability of the information contained herein and the potential
legal, compliance, tax, or accounting effects thereof. This document is not intended as investment research
or investment advice, or as a recommendation, offer, or solicitation for the purchase or sale of any security,
financial instrument, financial product or service, or to be used in any way for evaluating the merits of participating in any transaction.

\bibliography{refs.bib}

\clearpage
\appendix

\onecolumn

\section*{Appendix}

Appendix \ref{sec:fairpriv} provides background on privacy, the input perturbation notion, and how noise connects to differential privacy. Appendix \ref{sec:ci} provides more background on Chernoff Information and its connection to fairness. Appendix \ref{sec:proofs} reiterates useful Theorems and Lemmas from the Isotropic Gaussian examples as well as presents their proofs. Appendix \ref{appx:figures} provides a comparison between the fairness-accuracy curves and the log-fairness-accuracy curves for the initial Gaussian experiments. Additionally it holds a more comprehensive overview of Case 1 and 2 in the Gaussian setting. Appendix \ref{sec:consist} contains the proof of consistency for the Chernoff Information Neural Estimator. Appendix \ref{sec:exp_det} holds experimental details. Appendix \ref{sec:hsls} contains additional experiments on the HSLS dataset. Appendix \ref{sec:suppdat} provides a comparison of the fairness-accuracy curves and the log-fairness-accuracy curves for real world data. Finally, Appendix \ref{sec:ablation} provides a brief ablation over hyperparameter choices.

\renewcommand{\thesection}
{\Alph{section}}
\setcounter{section}{0}

\section{Background on Privacy} \label{sec:fairpriv}


\counterwithin{figure}{section}
\setcounter{figure}{0}

\counterwithin{theorem}{section}
\setcounter{theorem}{0}


In this work, we utilize the common notion of differential privacy.

\begin{definition} [Differential Privacy]
    A randomized learning algorithm \(\mathcal{A}\) is \((\varepsilon,\delta)\)-differentially private if, for every pair of neighboring datasets \(D\) and \(D'\) that differ in at most one element, and for every measurable subset \(S\) of the output space of \(\mathcal{A}\),
    
\begin{equation}
\Pr\!\bigl[\mathcal{A}(D) \in S\bigr]
\;\le\;
e^{\varepsilon}\,
\Pr\!\bigl[\mathcal{A}(D') \in S\bigr]
+
\delta.
\end{equation}
\end{definition}

Throughout this work, we utilize the input perturbation notion of privacy (Algorithm \ref{alg:noisy-sgd}).

\begin{algorithm}[h]
\caption{SGD with Gaussian Input Perturbation }
\label{alg:noisy-sgd}
\begin{algorithmic}[1]
\State \textbf{Input:} Dataset $D$, batch size $m$, learning rate $\alpha$, noise variance $\eta^2$, initial parameters $\bm{\theta_0}$
\State Create noisy dataset $\widetilde{D} =\{\mathbf{\tilde{x}_i}, y\}_{i=1}^n$ where $\mathbf{\tilde{x}_i} = \mathbf{x_i} + \bm{\xi_i}$ with $\bm{\xi_i} \sim \mathcal{N}(\bm{0}, \eta^2 \mathbf{I})$
\For{$t = 1, 2, \dots, T$}
    \State Sample a mini-batch $\{\mathbf{x_i}, y_i\}_{i=1}^m \sim \widetilde{D}$
    \State Compute stochastic gradient: \\
    ~~~~~~~~~$\mathbf{g_t} = \frac{1}{m} \sum_{i=1}^m \nabla_\theta \ell(\bm{\theta_{t-1}}; \mathbf{\tilde{x}_i}, y_i)$
    \State Update parameters: $\bm{\theta_t} \gets \bm{\theta_{t-1}} - \alpha \mathbf{g_t}$
\EndFor
\State \textbf{Return:} $\theta_T$
\end{algorithmic}
\end{algorithm}

For input perturbation, the naive bound on privacy stems from the feature level notion of neighboring datasets.

\begin{definition}[Neighboring Datasets -- Feature Level] \label{def:feature_neighbor}
Two datasets $D = \{(x_i,y_i)\}_{i=1}^n$ and $D' = \{(x'_i,y_i)\}_{i=1}^n$ are neighboring at the \emph{feature-level} if they differ in the features of a single element only. That is,

$$\exists\, j \in [n] \text{ such that } x_j \neq x'_j \text{ but } y_j=y'_j,
$$
and for all $i \neq j$, $x_i = x'_i$ and $y_i = y'_i$.
\end{definition}

The input perturbation notion of privacy leverages the Gaussian Mechanism to achieve differential privacy.

\begin{definition}[Gaussian Mechanism \citep{dwork2014algorithmic}]
Given a function $f : \mathcal{D} \to \mathbb{R}^d$ with $\ell_2$-sensitivity 
$\Delta_2(f) = \sup_{D, D'} \| f(D) - f(D') \|_2$ over all neighboring datasets $D, D'$, 
the Gaussian Mechanism with parameter $\eta > 0$ outputs
$$
\mathcal{M}(D) = f(D) + \mathcal{N}(0, \eta^2 \bm{I}_d),
$$
where $\mathcal{N}(0, \sigma^2 \bm{I}_d)$ is the $d$-dimensional Gaussian distribution with mean $0$ and covariance $\eta^2 \bm{I}_d$.
\end{definition}

It follows that the Gaussian Mechanism can provide differential privacy

\begin{theorem} [Theorem A.1 \citep{dwork2014algorithmic}]
     Let $\varepsilon \in (0,1)$ be arbitrary. For $c^2 \geq 2 \ln (1.25/\delta)$ the Gaussian Mechanism with parameter $\eta \geq c\Delta_2(f)/\varepsilon$ is $(\varepsilon, \delta)$-differentially private.
\end{theorem}

The input perturbation notion of privacy is equivalent to applying the Gaussian Mechanism to the identity function on our dataset and allows us to derive a relationship between privacy parameters $(\varepsilon, \delta)$ and the noise parameter $\eta^2$.

\begin{lemma}[\citet{dwork2014algorithmic}] \label{thm:dwork}
    Assume $\| \mathbf{x} \|_2 \leq 1$ for all $\mathbf{x} \in \mathbb{R}^d$. When $\eta^2 \geq \frac{8\log(1.25/\delta)}{\varepsilon^2}$, Algorithm 1 is $(\varepsilon, \delta$)-DP with respect to feature-level neighboring datasets (Definition \ref{def:feature_neighbor}).
\end{lemma}

\begin{proof}
    Since $\| \mathbf{x}\|_2 \leq 1$, the $\ell_2$-sensitivity of the identity function is bounded by $2$. Thus, by Theorem \ref{thm:dwork}, we get $(\varepsilon, \delta)$-differential privacy with noise $\eta^2 \geq \frac{8\log(1.25/\delta)}{\varepsilon^2}$. This $(\varepsilon, \delta)$-differential privacy extends to Algorithm 1 by the postprocessing property of differential privacy.
\end{proof}

Recent works have explored tightening this bound and generalizing to the more common notion of neighboring datasets.

\begin{definition}[Neighboring Datasets] \label{def:neighbor}
Two datasets $D = \{(x_i,y_i)\}_{i=1}^n$ and $D' = \{(x'_i,y'_i)\}_{i=1}^n$ are \emph{neighboring} if they differ in the data of a single element. That is,

$$\exists\, j \in [n] \text{ such that } (x_j,y_j) \neq (x'_j,y'_j),$$
and for all $i \neq j$, $(x_i,y_i) = (x'_i,y'_i)$.
\end{definition}

For $T$ training steps and $n$ samples, \citet{kang2020inputperturbationnewparadigm} show that the noise bound becomes $O(\frac{G^2T\log}{n(n-1)\sqrt{\Delta}\varepsilon^2})$ for a $G$-Lipschitz and $\Delta$-strongly convex loss function . \citet{fukuchi2017differentially}, show that for a quadratic loss function the noise bound becomes $O(\frac{G^2\log(16/\delta)) + \varepsilon}{n\varepsilon^2})$. In all cases, $\eta^2 \propto \frac{\log(1/\delta)}{\varepsilon^2}$. Thus, throughout this paper we utilize $\eta^2$ as our privacy parameter, since it 
provides a simple and standard expression directly linked to $(\varepsilon,\delta)$ 
via the Gaussian mechanism, allowing us to avoid restricting our analysis with assumptions on the learning algorithm. While tighter bounds can be derived, they do not affect the qualitative behavior of our results.

\section{Background on Chernoff Information} \label{sec:ci}


\counterwithin{figure}{section}
\setcounter{figure}{0}

\counterwithin{theorem}{section}
\setcounter{theorem}{0}


\subsection{Bounding the Error of Bayes Optimal Classifier}

\begin{lemma} 
    [\citep{nielsen2011}] \label{lem:bound2} For hypotheses $P_0(x)$ under $Y=0$ and $P_1(x)$ under $Y=1$, Chernoff Information bounds the error of the Bayes Optimal Classifier $H$:
    $$\Pr[H(X) \neq Y] \leq e^{-C(P_0, P_1)}.$$
\end{lemma}

\begin{proof}
    Note: Most of this proof is from \citet{nielsen2011}. We have included it for completeness.

    Consider the binary classification setting. Let $\Pr[Y = 0] = w_0 > 0$ and let $\Pr[Y = 1] = w_1 = 1- w_0 > 0$. The class probabilities can be defined as $p_0(x) = \Pr[x|Y = 0]$ and  $p_1(x) = \Pr[x|Y = 1]$. Now, consider the Bayes decision rule in which $H(x)$ classifies as $\hat{Y} = 0$ if $Pr[Y = 0|x] > Pr[Y = 1|x]$ and $\hat{Y} = 1$ otherwise. Using Bayes rule,  $\Pr[Y = i|x] = \frac{\Pr[Y = i]\Pr[x|Y = i]}{\Pr[x]} = \frac{w_ip_i(x)}{p(x)}$ for $i \in \{0,1\}$. Now, we can obtain the probability of error:

$$
\Pr(\text{Error} \mid x) =
\begin{cases}
\Pr(Y = 0 \mid x) & \text{if } H \text{ wrongly decided } \hat{Y} = 1, \\
\Pr(Y = 1 \mid x) & \text{if } H \text{ wrongly decided } \hat{Y} = 0.
\end{cases}.
$$

Thus, the Bayes decision rule minimizes by principle the average probability of error:

$$
E^* = \int \Pr[\text{Error}|x]p(x)dx = \int \min\{\Pr[Y = 0|x], \Pr[Y = 1|x]\}p(x)dx.
$$

Now, to upper bound this Bayes error, we can utilize knowledge that, for $a,b > 0$, $\min\{a,b\} \leq a^\alpha b^{1-\alpha} ~~\forall\alpha \in (0,1)$. This implies that:

$$
E^* \leq w_0^\alpha w_1^{1-\alpha}\int p_0(x)^\alpha p_1(x)^{1-\alpha}dx. 
$$

Now, let $c_\alpha(p_0 : p_1) = \int p_0(x)^\alpha p_1(x)^{1-\alpha}dx$. This term is referred to as the Chernoff $\alpha$-coefficient. Since the above inequality holds for all $\alpha \in (0,1)$, the best exponent for upper bounding the Bayes error corresponds to the optimal Chernoff $\alpha$-coefficient:

$$
c^*(p_0:p_1) = c_{\alpha^*}(p_0:p_1) = \inf_{\alpha \in (0,1)}\int p_0(x)^\alpha p_1(x)^{1-\alpha}dx
$$

The Chernoff Coefficient gives a measure of similarity which yields CI:

$$
C(p_0,p_1) = -\log \inf_{\alpha \in (0,1)}\int p_0(x)^\alpha p_1(x)^{1-\alpha}dx.
$$

 CI yields the best achievable exponent for a Bayesian probability of error:

$$
E^* \leq w_0^\alpha w_1^{1-\alpha}e^{-C(p_0, p_1)}.
$$

Lemma 1 follows as $w_0, w_1 < 1$.
\end{proof}

\subsubsection{Tightness of Bound} \label{sec:tightness}

CI provides an asymptotically tight bound on the error exponent, i.e.,

$$
\lim_{n \rightarrow \infty} -\frac{1}{n} \log P_e^{(n)} = C (P_0, P_1),
$$

where $n$ is the number of samples \citep{cover1999elements}. For instance, for Gaussian distributions, we can explicitly write:

$$
c_1 e^{-nC(P_0, P_1)} \leq P_e^{(n)} \leq c_2 e^{-nC(P_0, P_1)},
$$

which shows that CI provides an exponent-tight bound ($c_1$ and $c_2$ are constants with a closed-form expression).

\subsection{Connection to Fairness Metrics}

Next, we provide an overview of Chernoff Information and its connection to fairness following results from \citet{dutta2020there}.

Consider the binary classification task ($Y \in \{0,1\}$) with a binary protected attribute ($S \in \{0,1\}$). Consider the standard Bayes optimal classification rule of split classifiers: $H_s(x) = \frac{\Pr[Y = 1|S =s, x]}{\Pr[Y = 0|S =s, x]} \geq \tau_s$ in which the classifier predicts positive if the ratio of the conditionals is greater than a threshold $\tau_s$.

\begin{lemma}
    [Lemma 2 from \citet{dutta2020there}] The Chernoff exponent of the probability of error of the Bayes optimal classifier is given by the Chernoff Information:
    \begin{equation}
        E_e = C(P_0, P_1) = -\inf_{u\in(0,1)} \log(\int P_0(x)^{1-u}P_1(x)^udx). 
    \end{equation}
\end{lemma}

As highlighted in Lemma \ref{lem:bound}, this Chernoff exponent bounds the probability of error of the Bayes optimal classifier. We now show a connection between false positive and false negative error rates and Chernoff Information.

\begin{definition}
    [Chernoff Exponents---Definition 1 from \citet{dutta2020there}] The Chernoff exponents of the False positive error and False negative error are defined as 
    \begin{align}
         & E_\text{FP}(\tau_s) = \sup_{u >0}(u\tau_s - \Lambda_0(u)) \\
         & E_\text{FN}(\tau_s) = \sup_{u <0}(u\tau_s - \Lambda_1(u))
    \end{align}
    where $\Lambda_0$ and $\Lambda_1$ are the log-generating functions.
\end{definition}

Similarly, these Chernoff exponents bound the probability of false positive and false negative errors.

\begin{lemma}
    [Chernoff Bound - Lemma 1 from \citet{dutta2020there}] 
    The Chernoff exponents of false positive error and false negative error bound the probability of false positive and false negative error
    \begin{align}
        P_\text{FP} \leq e^{-E_{FP}(\tau_s)}\\
        P_\text{FN} \leq e^{-E_{FN}(\tau_s)}
    \end{align}
\end{lemma}

Under the equal prior condition, the overall Chernoff exponent can be written as simple function of the false positive and false negative exponents.
\begin{definition}
    [Definition 2 from \citet{dutta2020there}] Suppose $\Pr[S = 0] = \Pr[S=1]$ and $\Pr[Y=1|S =s] = \Pr[Y=0|S=s]$ for all $s$. The Chernoff exponent of the overall probability of error $P_e$ is defined as
    \begin{equation}
        E_e = \min\{E_{FP}(\tau_s), E_{FN}(\tau_s)\}
    \end{equation}
\end{definition}

Thus, Chernoff Difference naturally measures disparities in error exponents across groups (False positive or False negative exponents). 

\begin{equation}
\text{CD} = |\min\{E_{FP}(\tau_P), E_{FN}(\tau_P)\} - \min\{E_{FP}(\tau_Q), E_{FN}(\tau_Q)\}|    
\end{equation}

In the asymptotic regime, CD expresses disparities in terms of the maximum error rate on a log-scale, capturing a geometric difference rather than the arithmetic difference used in common fairness metrics such as Equal Opportunity. Next, we demonstrate how Chernoff Difference relates to the standard arithmetic difference.

\begin{theorem} \label{thm:lipschitz}
    Let $M_p = \max\{P_{FP}^P, P_{FN}^P\}$ and $M_q = \max\{P_{FP}^Q, P_{FN}^Q\}$. For all constants $0< c < C$ with $c \leq M_P, M_Q \leq C$, we have
    \begin{equation}
        c\left|\log M_p - 
    \log M_Q\right| \leq \left|M_p - M_Q\right|\leq C\left|\log M_p - 
    \log M_Q\right|.
    \end{equation}
\end{theorem}

\begin{proof}
    Since $M_p,M_q \in [c,C]$ with $c>0$, the Mean Value Theorem implies
\begin{equation}
    \log M_p - \log M_q = \log'( \xi )(M_p-M_q)=\frac{1}{\xi}(M_p-M_q)    
\end{equation}

for some $\xi$ between $M_p$ and $M_q$. Thus
\begin{equation}
    |M_p-M_q|=\xi\,|\log M_p-\log M_q|.    
\end{equation}

Since $\xi \in [c,C]$, we have
\begin{equation}
c\,|\log M_p-\log M_q|
\;\le\;
|M_p-M_q|
\;\le\;
C\,|\log M_p-\log M_q|.
\end{equation}
\end{proof}

When we break the equal prior assumption the Chernoff Information continues to capture error exponents, however skewed by the class priors.

\begin{definition}
    [Section E.1 \citep{dutta2020there}] Let $\pi_0 = P[Y=0]$ and $\pi_1 = P[Y=1]$. The Chernoff exponent of the overall probability of error $P_e$ is defined as
    \begin{equation}
        E_e = \min\{E_{FP}(\tau_s) - \log2\pi_0, E_{FN}(\tau_s) - \log 2\pi_1\}
    \end{equation}
\end{definition}

Thus, the Chernoff Difference continues to evaluate the difference of the error exponents across groups. In the case where where the exponents align as false negative exponents, we again achieve a similar value to equal opportunity (although skewed by the log-ratio of the priors).

\begin{equation}
    \text{CD} = |\min\{E_{FP}(\tau_P)  - \log2\pi^{(P)}_0, E_{FN}(\tau_P)- \log 2\pi^{(P)}_1\} - \min\{E_{FP}(\tau_Q)- \log2\pi^{(Q)}_0, E_{FN}(\tau_Q)- \log 2\pi^{(Q)}_1\}|    
\end{equation}

\subsection{Multi-group Fairness}\label{ssec:multigroup}

While we focus on the binary setting, we can generalize our framework to multi-group settings. In these settings, Chernoff Information can play the role of a group statistic (analogous to group accuracy or group false-positive rate), which can be compared across all groups. The most natural generalization is to compute the maximum Chernoff Difference across all groups:
\begin{equation}
\mathrm{CD} = \max_{i \neq j \in \mathcal{S}} \left| C\!\left(P^{(i)}_0, P^{(i)}_1\right) - C\!\left(P^{(j)}_0, P^{(j)}_1\right) \right|,
\end{equation}
where the superscript denotes different sensitive groups.

This worst-case Chernoff Difference then captures the largest disparity in classification difficulty across all groups. This corresponds to the maximum notion of equal opportunity:
\begin{equation}
\mathrm{EO}_{\max} =
\max_{i \neq j \in \mathcal{S}} \lvert \mathrm{FNR}_i - \mathrm{FNR}_j \rvert,
\end{equation}
which is a common way to define equal opportunity for multi-group settings \cite{hardt2016equality, alghamdi2022beyond}.

\section{Isotropic Gaussian Proofs} \label{sec:proofs}

\counterwithin{figure}{section}
\setcounter{figure}{0}

\counterwithin{theorem}{section}
\setcounter{theorem}{0}
\begin{lemma}
    [\textbf{Restatement of Lemma \ref{lemma}}]
    When $P_0(x)\sim\mathcal{N}(\mu_0,\sigma^2\mathbf{I})$, $P_1(x)\sim\mathcal{N}(\mu_1,\sigma^2\mathbf{I})$, $Q_0(x)\sim\mathcal{N}(\zeta_0,\tau^2\mathbf{I})$, and $Q_1(x)\sim\mathcal{N}(\zeta_1,\tau^2\mathbf{I})$, the Chernoff Difference is given as:
    $$CD = \left|\frac{\|\mu_0 - \mu_1\|_2^2}{8\sigma^2} - \frac{\|\zeta_0 - \zeta_1\|_2^2}{8\tau^2}\right|.$$ 
\end{lemma}

\begin{proof}
    Recall the definition of Chernoff Difference and Chernoff Information

    \begin{align}
        CD & = |C(P_0, P_1) - C(Q_0, Q_1)| \\ &= |\min_{u\in(0,1)}\log \int P_0(x)^uP_1(x)^{1-u}dx - \min_{v\in(0,1)}\log \int Q_0(x)^vQ_1(x)^{1-v}dx|.
    \end{align}

    Now, following a result from \citet{dutta2020there}, we see that for $P_0\sim\mathcal{N}(\mu_0,\sigma^2\mathbf{I})$, $P_1\sim\mathcal{N}(\mu_1,\sigma^2\mathbf{I})$.
    \begin{align}
        &\log \int P_0(x)^uP_1(x)^{1-u}dx \\&=  \log \int e^{-\frac{u}{2\sigma^2} \left( (x-\mu_1)^T(x-\mu_1) - (x-\mu_0)^T(x-\mu_0) \right)} P_0(x) \, dx \\
&= \log e^{-\frac{u}{2\sigma^2} \left( \mu_1^T \mu_1 - \mu_0^T \mu_0 \right)} 
    \int e^{-\frac{u}{2\sigma^2} \left( -2x^T(\mu_1 - \mu_0) \right)} P_0(x) \, dx \\
&= \log e^{-\frac{u}{2\sigma^2} \left( \mu_1^T \mu_1 - \mu_0^T \mu_0 \right)} 
    e^{-\frac{u}{2\sigma^2} \left( -2\mu_0^T(\mu_1 - \mu_0) \right)} 
    e^{\frac{u^2}{2\sigma^2} \left( \|\mu_1 - \mu_0\|_2^2 \right)} \\
&= \log e^{-\frac{u}{2\sigma^2} \left( \|\mu_1 - \mu_0\|_2^2 \right)} 
    e^{\frac{u^2}{2\sigma^2} \left( \|\mu_1 - \mu_0\|_2^2 \right)}\\& = \frac{u(u-1)}{2\sigma^2}\|\mu_1 - \mu_0\|_2^2.
    \end{align}

    Now, we can compute the derivative to minimize with respect to $u$:

    \begin{align}
        \frac{d}{du}\frac{u(u-1)}{2\sigma^2}\|\mu_1 - \mu_0\|_2^2 = \frac{2u-1}{2\sigma^2}\|\mu_1 - \mu_0\|_2^2.
    \end{align}

    The derivative is 0, when $u = 0.5$. This critical point is a minimum following a first derivative test. Thus, 

    \begin{align}
        C(P_0, P_1) = \frac{\|\mu_0 - \mu_1\|_2^2}{8\sigma^2}.
    \end{align}

    When these distributions are 1-dimensional, they reduce to the closed form derived in \citet{nielsen2022revisiting}. A similar approach can be done for $Q_0(x)$ and $Q_1(x)$. Thus, the CD for Gaussian distributions is defined as

    \begin{align}
        CD = |\frac{\|\mu_0 - \mu_1\|_2^2}{8\sigma^2} - \frac{\|\zeta_0 - \zeta_1\|_2^2}{8\tau^2}|.
    \end{align}
    
\end{proof}

\begin{theorem}[ \textbf{Restatement of Theorem \ref{thm:central}}]
    Suppose $P_0(x)\sim\mathcal{N}(\mu_0,\sigma^2\mathbf{I})$, $P_1(x)\sim\mathcal{N}(\mu_1,\sigma^2\mathbf{I})$, $Q_0(x)\sim\mathcal{N}(\zeta_0,\tau^2\mathbf{I})$, and $Q_1(x)\sim\mathcal{N}(\zeta_1,\tau^2\mathbf{I})$. Without loss of generality, we assume that $\|\mu_0 - \mu_1\|_2 \geq \|\zeta_0 - \zeta_1\|_2$. There are three behaviors of the Noisy Chernoff Difference ($\widetilde{\text{CD}}_{\eta^2}$): (i) $\widetilde{\text{CD}}_{\eta^2}$ has a maximum point, (ii) $\widetilde{\text{CD}}_{\eta^2}$ has a maximum point \underline{and} a reflection point (where $\widetilde{\text{CD}}_{\eta^2} = 0$), (iii) $\widetilde{\text{CD}}_{\eta^2}$ is non-increasing.\footnote{When $\|\mu_0 - \mu_1\|_2 = \|\zeta_0 - \zeta_1\|_2$, $\widetilde{\text{CD}}_{\eta^2}$ will always fall into this case.} The respective conditions for these three cases are given as follows:
    \begin{enumerate}
    \item[]
        \item[(i)] $\frac{\|\zeta_0 - \zeta_1\|^2_2}{\|\mu_0 - \mu_1\|^2_2} <\frac{\tau^2}{\sigma^2} < \frac{\|\zeta_0 - \zeta_1\|_2}{\|\mu_0 - \mu_1\|_2} < 1$,  \hfill \textbf{(Case 1: Maximum Point)}

    \item[]     
        \item [(ii)] $\frac{\tau^2}{\sigma^2} < \frac{\|\zeta_0 - \zeta_1\|^2_2}{\|\mu_0 - \mu_1\|^2_2} < 1$, \hfill \textbf{(Case 2: Maximum and Reflection)}
        
    \item[]  
        \item [(iii)] Neither condition (i) or (ii) hold. \hfill \textbf{(Case 3: Non-increasing)}
    \end{enumerate}    
\end{theorem}

\begin{proof}

 Recall the definition of noisy Chernoff Difference. We define a signed noisy Chernoff Difference $s\widetilde{\text{CD}}_{\eta^2}$ such that $|s\widetilde{\text{CD}}_{\eta^2}| = \widetilde{\text{CD}}_{\eta^2}$. Thus,

\begin{align}
      s \widetilde{\text{CD}}_{\eta^2}=
     \frac{1}{8(\tau^2 + \eta^2)}\|\zeta_0 - \zeta_1\|^2_2 - \frac{1}{8(\sigma^2+\eta^2)}\|\mu_0 - \mu_1\|^2_2.
\end{align}

\paragraph{Case 1 and 2.}Let $p = \|\mu_0 - \mu_1\|_2$ and let $q = \|\zeta_0 - \zeta_1\|_2$. First, we can analyze the positive $\eta^2$ regime for a critical point. To find potential critical points, consider the derivative of $s\widetilde{\text{CD}}_{\eta^2}$.

\begin{align}
    \frac{\partial}{\partial\eta^2} s \widetilde{\text{CD}}_{\eta^2}  = \frac{p^2}{8(\sigma^2 + \eta^2)^2} - \frac{q^2}{8(\tau^2 + \eta^2)^2} = \frac{p^2(\tau^2 + \eta^2)^2-q^2(\sigma^2 + \eta^2)^2}{8(\sigma^2 + \eta^2)^2(\tau^2 + \eta^2)^2}
\end{align}

Now, the potential critical points will be $\eta^2$ where $s\widetilde{\text{CD}}_{\eta^2}' = 0$. That is $0 = p^2(\tau^2 + \eta^2)^2-q^2(\sigma^2 + \eta^2)^2$. Thus, the critical points are:

\begin{align}
    \eta_0^2 = \frac{\sigma^2q-\tau^2p}{p-q} \\\eta_1^2 =\frac{-\sigma^2q-\tau^2p}{p+q}
\end{align}

However, we observe that $\eta_1^2$ is always negative, thus the only useful critical point in the postive $\eta^2$ region is $\eta_0^2$. By our assumption, we know that $p-q \geq 0$. First, we observe that there is no critical point when $p=q$ as $\eta_0^2$ does not exist. Thus, $\eta_0^2$ is positive when the following conditions hold.

\begin{align}
    \sigma^2q-\tau^2p >0 \\
    \frac{\tau^2}{\sigma^2} < \frac{\|\zeta_0 - \zeta_1\|_2}{\|\mu_0 - \mu_1\|_2} < 1
\end{align}

Next, we will show that this critical point is always a maximum in the positive $\eta^2$ regime.

First, consider the second derivative of the signed noisy Chernoff difference.

\begin{align}
    \frac{\partial^2}{\partial (\eta^2)^2}s\widetilde{\text{CD}}_{\eta^2} =  \frac{q^2}{4(\tau^2 + \eta^2)^3} - \frac{p^2}{4(\sigma^2 + \eta^2)^3}.
\end{align}

Plugging in the relevant critical point we observe that

\begin{align}
    \frac{\partial^2}{\partial (\eta^2)^2 }s  \widetilde{\text{CD}}_{\eta_0^2} &=  \frac{q^2}{4(\tau^2 + \eta_0^2)^3} - \frac{p^2}{4(\sigma^2 + \eta_0^2)^3} = \frac{q^2}{4(\frac{q(\sigma^2-\tau^2)}{p-q})^3} -  \frac{p^2}{4(\frac{p(\sigma^2-\tau^2)}{p-q})^3} \\&=  \frac{q^2(p-q)^3}{4q^3(\sigma^2-\tau^2)^3} - \frac{p^2(p-q)^3}{4p^3(\sigma^2-\tau^2)^3} = \frac{(p-q)^4}{4p^3q^3(\sigma^2-\tau^2)^3}.
\end{align}
Now, we know that $\sigma^2 > \tau^2$ so the second derivative of  this critical point of $s\widetilde{\text{CD}}_{\eta^2}$ must be a minimum. However, the goal is to examine the behavior of $\widetilde{\text{CD}}_{\eta^2}$. So, we can show that this is critical point is a maximum of $\widetilde{\text{CD}}_{\eta^2}$, by showing that $s\widetilde{\text{CD}}_{\eta_0^2}$ is negative. By plugging in the critical point, we can observe that it is a maximum of $\widetilde{\text{CD}}_{\eta^2}$.

\begin{align}
    s\widetilde{\text{CD}}_{\eta_0^2} &= \frac{q^2}{8(\tau^2 + \eta_0^2)} - \frac{p^2}{8(\sigma^2 + \eta_0^2)} = \frac{q^2}{8(\tau^2 + \frac{\sigma^2q-\tau^2p}{p-q})} - \frac{p^2}{8(\sigma^2 + \frac{\sigma^2q-\tau^2p}{p-q})} \\&= \frac{(p-q)}{8}(\frac{q}{\sigma^2-\tau^2}-\frac{p}{\sigma^2-\tau^2}) = \frac{(p-q)(q-p)}{8(\sigma^2 - \tau^2)}
\end{align}

Now, this value is always negative as we know $p > q$ and $\sigma^2 > \tau^2$. Thus, the positive critical point is a maximum. Now, we can analyze the positive $\eta^2$ regime for a reflection point. 

\begin{align}
    &s\widetilde{\text{CD}}_{\eta^2} = \frac{q^2}{8(\tau^2 + \eta^2)} - \frac{p^2}{8(\sigma^2 + \eta^2)} \\ \notag
    \\
    &8q^2(\sigma^2 + \eta^2) - 8p^2(\tau^2 + \eta^2) = 0 \\ \notag
    \\
    &\eta^2 = \frac{q^2\sigma^2-p^2\tau^2}{p^2-q^2} = \frac{\|\zeta_0 - \zeta_1\|_2^2\sigma^2-\|\mu_0 - \mu_1\|_2^2\tau^2}{\|\mu_0 - \mu_1\|_2^2-\|\zeta_0 - \zeta_1\|_2^2}.
\end{align}

Now, the denominator of this is always positive, thus, $\eta^2$ is positive when the following holds:

\begin{align}
    \|\zeta_0 - \zeta_1\|_2^2\sigma^2-|\mu_0 - \mu_1\|_2^2\tau^2 > 0
    \\ \frac{\tau^2}{\sigma^2} < \frac{\|\zeta_0 - \zeta_1\|^2_2}{\|\mu_0 - \mu_1\|^2_2} < 1.
\end{align}

\paragraph{Case 3.} Finally, we can examine the behavior when none of these conditions hold. Suppose $\frac{\tau^2}{\sigma^2} \geq \frac{\|\zeta_0 - \zeta_1\|_2}{\|\mu_0 - \mu_1\|_2} = \frac{q}{p}$. We can analyze the sign of the numerator of $s\widetilde{\text{CD}}_{\eta^2}'$ by writing it as

\begin{align}
    &p^2(\tau^2 + \eta^2)^2 - q^2(\sigma^2 + \eta^2)^2 \\
            &= (p^2(\tau^2 + \eta^2) - q^2(\sigma^2 + \eta^2))(p^2(\tau^2 + \eta^2) + q^2(\sigma^2 + \eta^2)).
\end{align}

Thus, to analyze the sign, we can analyze $p^2(\tau^2 + \eta^2) - q^2(\sigma^2 + \eta^2)$. We observe

\begin{align}
    p^2(\tau^2 + \eta^2) - q^2(\sigma^2 + \eta^2) = p^2\tau^2 - q^2\sigma^2 + \eta^2(p^2 - q^2).
\end{align}

Now, from our initial assumption, $p^2 - q^2 \geq 0$. From the assumption that $\frac{\tau^2}{\sigma^2} \geq  \frac{q}{p}$, $p^2\tau^2 - q^2\sigma^2 \geq 0$. Thus, the sign is always positive. Now, to show that $\widetilde{\text{CD}}_{\eta^2}$ is non-increasing, we will show $s\widetilde{\text{CD}}_{\eta^2}$ is not positive. 

\begin{align}
    s\widetilde{\text{CD}}_{\eta^2} &= \frac{q^2}{8(\tau^2 + \eta^2)} - \frac{p^2}{8(\sigma^2 + \eta^2)}  = \frac{q^2(\sigma^2 + \eta^2) - p^2(\tau^2 + \eta^2)}{8((\tau^2 + \eta^2))(\sigma^2 + \eta^2)}
\end{align}

Now, we analyze $q^2(\sigma^2 + \eta^2) - p^2(\tau^2 + \eta^2)$. We can see that this is equivalent to $q^2\sigma^2 - p^2\tau^2 + \eta^2(q^2 - p^2)$. From our initial assumption, $q^2 - p^2 \leq 0$. From the assumption that $\frac{\tau^2}{\sigma^2} \geq \frac{q}{p}$, $q^2\sigma^2 - p^2\tau^2 \leq 0$. Thus, $s\widetilde{\text{CD}}_{\eta^2}$ is always negative and $\widetilde{\text{CD}}_{\eta^2}$ is non-increasing over the positive $\eta^2$ regime.
\end{proof}

\newpage
\section{Supplemental Gaussian Figures} \label{appx:figures}

\subsection{Experimental Details}
\label{ssec:gauss_det}
For the synthetic Gaussian experiments, we use 1-D Gaussian distributions for $P_0(x), P_1(x), Q_0(x),$ and $Q_1(x)$ and get 100,000 i.i.d. samples from each distribution to create a balanced dataset. We train a Gaussian Na\"{i}ve Bayes classifier for each of the groups (this provides the Bayes optimal classifier). We measure the overall accuracy of our classifier and quantify fairness by using the true positive rate disparity between groups (as this corresponds to the dominating exponent). To achieve fairness-accuracy trade-offs, we choose different prior beliefs of our labels for the Gaussian Na\"{i}ve Bayes classifier. By perturbing the prior probabilities for the unprivileged group, we create fairness-accuracy curves for each of our settings\footnote{A proxy for adjusting the threshold of the Bayes optimal classifier \citep{dutta2020there}.}.

\subsection{Fairness-Accuracy vs Log Fairness-Accuracy} \label{ssec:logcurves}

\begin{figure}[!ht]
  \centering

  \begin{subfigure}{0.48\columnwidth}
    \centering
    \includegraphics[width=\linewidth]{imgs/figureA1.pdf}
    \caption{}
    \label{fig:caseb1a}
  \end{subfigure}
  \hfill
  \begin{subfigure}{0.48\columnwidth}
    \centering
    \includegraphics[width=\linewidth]{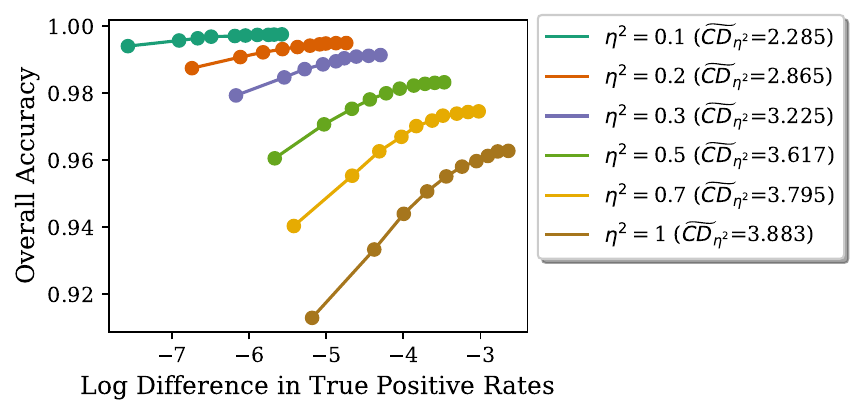}
    \caption{}
    \label{fig:caseb1b}
  \end{subfigure}

  \caption{{\footnotesize \textbf{(Case 1: Privacy Hurts Fairness)} $\mu_0 = 0, \mu_1 = 16.5, \sigma = 2.43$, $\zeta_0 = 0.5, \zeta_1 = 3.8, \tau = 0.55$. (a) Fairness-Accuracy Curve. (b) Log Fairness-Accuracy Curve. We observe a steepening effect.}}
\end{figure}

\begin{figure}[ht]
  \centering
  \begin{subfigure}{0.48\columnwidth}
    \centering
    \includegraphics[width=\linewidth]{imgs/figure2b.pdf}
    \caption{}
    \label{fig:caseb2a}
  \end{subfigure}
  \hfill
  \begin{subfigure}{0.48\columnwidth}
    \centering
    \includegraphics[width=\linewidth]{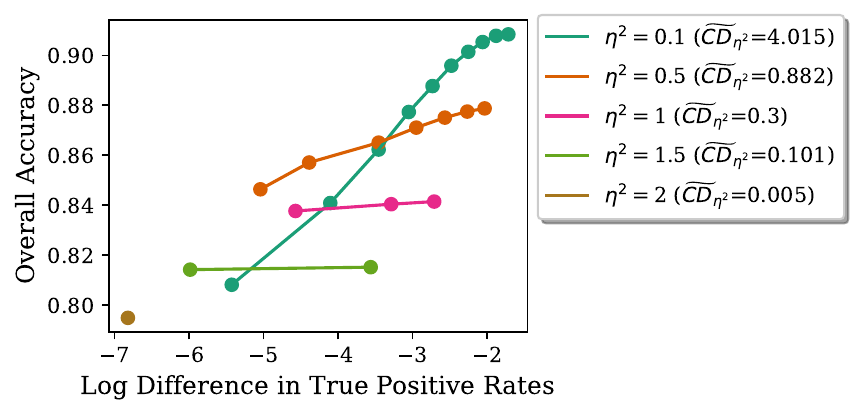}
    \caption{}
    \label{fig:caseb2b}
  \end{subfigure}
  \caption{{\footnotesize \textbf{(Case 2: Privacy Can Give Free Fairness)} $\mu_0 = -4.2, \mu_1 = 1.3, \sigma = 3$, $\zeta_0 = 0.3, \zeta_1 = 2.7, \tau = 0.25$. (a) Fairness-Accuracy Curve. (b) Log Fairness-Accuracy Curve. We observe a flattening effect.}}
\end{figure}

\begin{figure}[!ht]
  \centering
  \begin{subfigure}{0.48\columnwidth}
    \centering
    \includegraphics[width=\linewidth]{imgs/figure3b.pdf}
    \caption{}
    \label{fig:caseb3a}
  \end{subfigure}
  \hfill
  \begin{subfigure}{0.48\columnwidth}
    \centering
    \includegraphics[width=\linewidth]{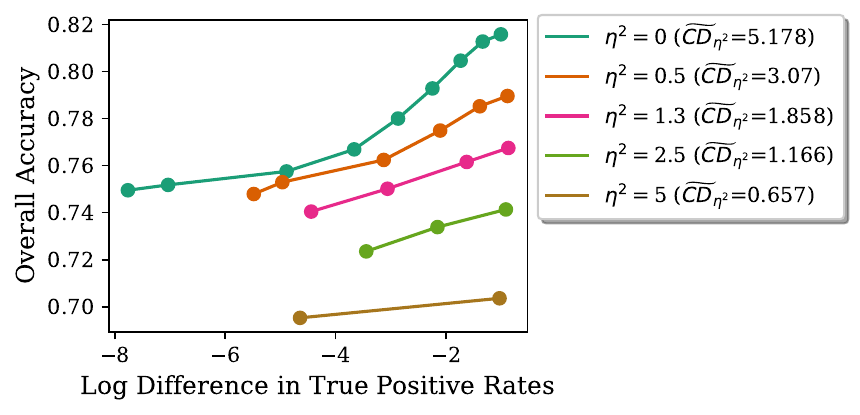}
    \caption{}
    \label{fig:caseb3b}
  \end{subfigure}
  \caption{{\footnotesize \textbf{(Case 3: Triple Trade-off)} $\mu_0 = -4.2, \mu_1 = 1.3, \sigma = 0.85$, $\zeta_0 = 0.6, \zeta_1 = 1.6, \tau = 0.6$. (a) Fairness-Accuracy Curve. (b) Log Fairness-Accuracy Curve. We observe a much slower flattening effect.}}
\end{figure}
\newpage

\subsection{Case 1: Extended plots}\label{ssec:case1ext}

\begin{figure}[!ht]
  \centering
  \begin{subfigure}{0.48\columnwidth}
    \centering
    \includegraphics[width=\linewidth]{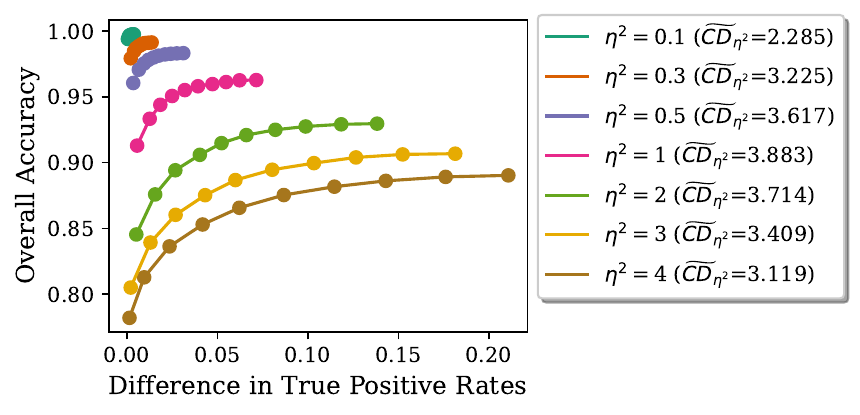}
    \caption{}
    \label{fig:a1appx}
  \end{subfigure}
  \hfill
  \begin{subfigure}{0.48\columnwidth}
    \centering
    \includegraphics[width=\linewidth]{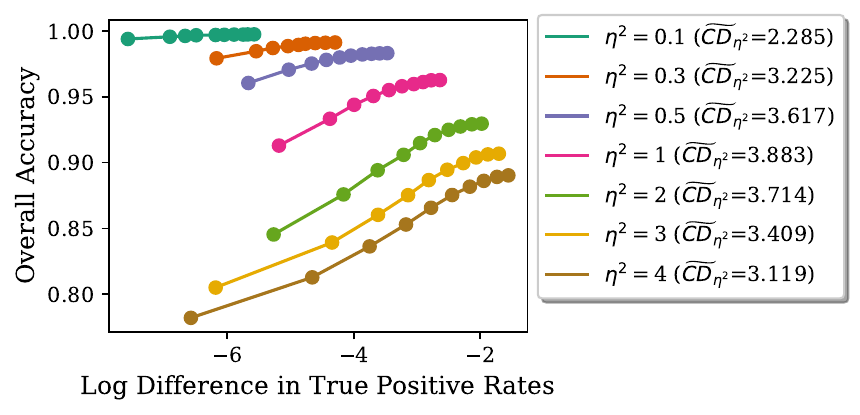}
    \caption{}
    \label{fig:a1logappx}
  \end{subfigure}
  \caption{{\footnotesize \textbf{(Case 1: Extension)} $\mu_0 = 0, \mu_1 = 16.5, \sigma = 2.43$, $\zeta_0 = 0.5, \zeta_1 = 3.8, \tau = 0.55$ (a) Fairness-Accuracy Curve. (b) Log Fairness-Accuracy Curve. After reaching the maximum CD, the fairness-accuracy curves begin to slowly flatten (nearly imperceptible), reflecting the slow decay in CD.}}
\end{figure}
\subsection{Case 2: More detail} \label{ssec:case2_moredetail}

\begin{figure}[!ht]
  \centering
  \begin{subfigure}{0.45\columnwidth}
    \centering
    \includegraphics[width=\linewidth]{imgs/figure2a.pdf}
    \caption{}
    \label{fig:caseb21}
  \end{subfigure}
  \begin{subfigure}{0.45\columnwidth}
    \centering
    \includegraphics[width=\linewidth]{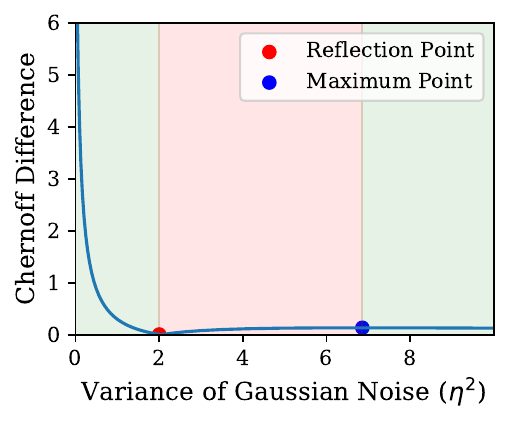}
    \caption{}
    \label{fig:casebb22}
  \end{subfigure}
  \caption{{\footnotesize \textbf{(Case 2: Privacy Can Give Free Fairness)} $\mu_0 = -4.2, \mu_1 = 1.3, \sigma = 3$, $\zeta_0 = 0.3, \zeta_1 = 2.7, \tau = 0.25$. (a) Initial plot of $\widetilde{\text{CD}}_{\eta^2}$. (b) Full plot showing presence of maximum point.}}
\end{figure}

\begin{figure}[!ht]
  \centering
  \begin{subfigure}{0.48\columnwidth}
    \centering
    \includegraphics[width=\linewidth]{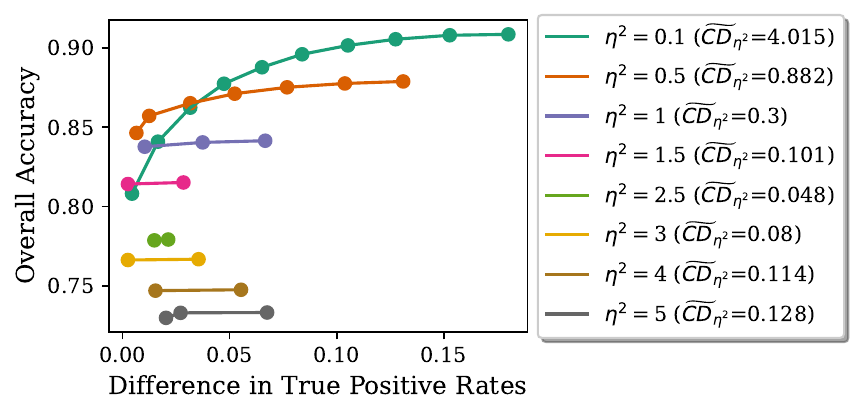}
    \caption{}
    \label{fig:a2appx}
  \end{subfigure}
  \hfill
  \begin{subfigure}{0.48\columnwidth}
    \centering
    \includegraphics[width=\linewidth]{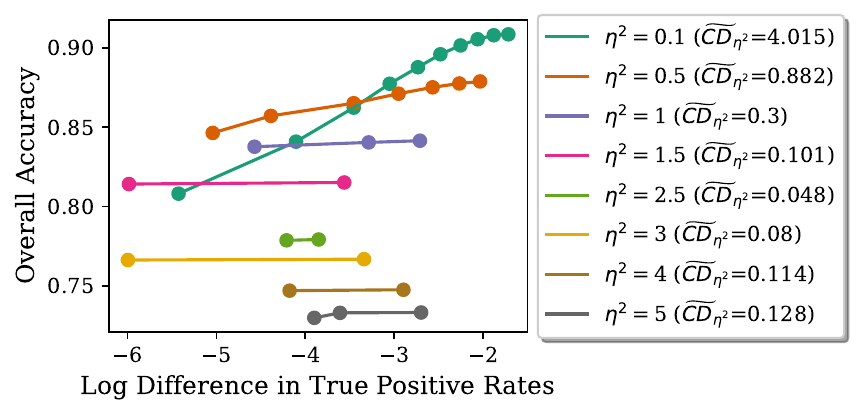}
    \caption{}
    \label{fig:a2logappx}
  \end{subfigure}
  \caption{{\footnotesize \textbf{(Case 2: Extension)} (a) Fairness-Accuracy Curve. (b) Log Fairness-Accuracy Curve. After reaching the reflection point, CD begins to increase. However this spike is so small any change in slope in the fairness-accuracy curve is imperceptible.}}
\end{figure}

\section{Consistency}\label{sec:consist}

\subsection{Setting}
Let $r(x) = \frac{P_1(x)}{P_0(x)}$  be the density ratio and let $f(u) = -\log \mathbb{E}_{x \sim P_0}[r(x)^u]$. We also define the corresponding empirical estimators $r_n(x)$ and $f_n(u) = -\log \frac{1}{n}\sum r_n(x_i)^u$. We operate under the assumption that $P_0$ and $P_1$ share a support, $r(x)$ is bounded i.e., $r(x) \in [c, C]$ for some constants $0 < c< C< \infty$.

We further assume that $r_n(x)$ is a consistent estimator of $r(x)$ and is bounded via a clipping mechanism i.e., $r_n(x) \in [c,C]$. Clipping is typically applied after estimation to prevent numerical instability and prevent estimates from growing too large or small. This clipping does not affect consistency under our assumptions if the clipping bound is wider than the true bounds on the density ratio. Finally, we assume that the estimator $r_n(x)$ is trained on its own samples, independent of the evaluation points.

\subsection{Proofs}

We will use Newey-McFadden's main consistency theorem to show consistency of the Chernoff Information estimator.

\begin{theorem}[\textbf{Newey-McFadden Main Consistency Theorem}, Theorem 2.1 \citep{newey1994large}]\label{thm:newmcf} If there is a function f such that (i) $f(u)$ is uniquely maximized at $u^*$; (ii) $U$ is compact; (iii) $f$ is continuous; (iv) $f_n(u)$ converges uniformly in probability to $f(u)$, then $\hat{u}\rightarrow u^*$.
\end{theorem}

We begin with a series of lemmas before completing the proof of our theorem. First, we are able to show compactness via: 

\begin{lemma} \label{lem:closed}
    We can write Chernoff Information as $-\inf_{[0,1]}\log(\int P_0(x)^{1-u}P_1(x)^udx)$.
\end{lemma}

\begin{proof}
    Computing the infimum over the open interval (0,1) and closed interval [0,1] yield the same result following the strict convexity of 
$f$. We point to Remark 2 from \citet{nielsen2022revisiting} for a discussion of the boundary conditions. In fact, in some references, Chernoff Information is presented using this closed interval \citep{cover1999elements}. This provides a compactness property that can be utilized in the following lemmas.
\end{proof}

\begin{lemma} \label{ref:pointwise}
    $\{f_n\}$ converges pointwise in probability for $u \in [0,1]$
\end{lemma}

\begin{proof}

First, we see that we can decompose the following difference into two terms,

\begin{equation}
   \frac{1}{n}\sum r_n(x_i)^u - \mathbb{E}_{x \sim P_0}[r(x)^u] = (\frac{1}{n}\sum r_n(x_i)^u - \mathbb{E}_{x \sim P_0}[r_n(x)^u]) + (\mathbb{E}_{x \sim P_0}[r_n(x)^u]- \mathbb{E}_{x \sim P_0}[r(x)^u])
\end{equation}

Next, we can see that the first term $\frac{1}{n}\sum r_n(x_i)^u - \mathbb{E}_{x \sim P_0}[r_n(x)^u]$ converges to 0 almost surely (a.s.) by the law of large numbers.

 Since $r_n$ is consistent and uniformly bounded in $[c, C]$, there exists $\widetilde{r}_n(x)^u$ and $\widetilde{r}(x)^u$ defined on another space such that $\widetilde{r}_n(x)^u$ has the same distribution as $r_n(x)^u$, $\widetilde{r}(x)^u$ has the same distribution as $r(x)^u$, and $\widetilde{r}_n(x)^u$ converges to $\widetilde{r}(x)^u$ a.s. (via the Continuous Mapping Theorem and Skorohod's Representation Theorem). The Dominated Convergence Theorem then implies $\mathbb{E}[\widetilde{r}_n(x)^u]- \mathbb{E}[\widetilde{r}(x)^u]$ converges to 0. Since the two sequences share the same distributions, convergence in the Skorokhod space implies convergence in the original space. It follows that $\mathbb{E}[r_n(x)^u]- \mathbb{E}[r(x)^u]$ converges to 0.

Since both terms converge to 0, their sum converges a.s. to 0 and hence in probability to 0. Thus, we have $\frac{1}{n}\sum r_n(x_i)^u - \mathbb{E}[r(x)^u] \rightarrow 0$ in probability. By applying the continuous mapping theorem, we have $-\log(\frac{1}{n}\sum r_n(x_i)^u) + \log(\mathbb{E}[r(x)^u]) \rightarrow 0$. Thus, we have $\{f_n\}$ converges pointwise in probability.
\end{proof}

\begin{lemma} \label{lem:equicontinuity}
    $\{f_n\}$ is stochastically equicontinuous on $[0,1]$
\end{lemma}

\begin{proof}

First, recall the density ratio $r(x)$ is bounded between $[c, C]$. Via clipping we assume the estimator $r_n(x)$ follows the same bound. Now, for any $x$, the function $u \rightarrow r_n(x)^u$ has a bounded derivative, $r_n(x)^u \log r_n(x)$, that only depends on constants $c$ and $C$, thus is Lipschitz in $u$. It follows that the average 
$$
\frac{1}{n}\sum_{i=1}^n r_n(x_i)^u
$$
is Lipschitz in $u$ again with the constant dependent only on $c$ and $C$. Now, since $-\log(\cdot)$ is Lipschitz on the relevant range (away from 0), it follows from the compositional property of Lipschitz functions that
$$
f_n(u) = -\log\Big(\tfrac{1}{n}\sum_{i=1}^n r_n(x_i)^u\Big)
$$
is Lipschitz in $u$, with constant only dependent on $c$ and $C$. Thus, we have a uniform Lipschitz bound across all elements of the sequenece $\{f_n\}$. It follows that $\{f_n\}$ is stochastically equicontinuous on $[0,1]$.
\end{proof}

\begin{lemma} \label{lem:uniformly}
    $\{f_n\}$ converges uniformly in probability for $u \in [0,1]$ 
\end{lemma}

\begin{proof}
    This follows from Theorem 1 from \citet{newey1991uniform} as well as pointwise convergence in Lemma \ref{ref:pointwise} and stochastic equicontinuity in Lemma \ref{lem:equicontinuity}.
\end{proof}
    
Finally, using these lemmas, we can complete the proof of our theorem.

\begin{theorem}[Restatement of Theorem 2] Let $P_0$ and $P_1$ be distributions with common support and bounded density ratio. Then the Chernoff Information estimator, constructed using a consistent density ratio estimator, is itself consistent.
\end{theorem}

\begin{proof}
    This result follows from Theorem \ref{thm:newmcf}. Continuity and the uniqueness of the maximizer are given by \citet{nielsen2022revisiting}. Compactness is given by Lemma \ref{lem:closed}. Uniform Convergence in probability is given by Lemma \ref{lem:uniformly}. Hence, 
    \begin{equation}
        \widehat{\text{CI}} \rightarrow \text{CI}.
    \end{equation}

    Therefore, we have consistency of the Chernoff Information estimator. 
\end{proof}

\section{Experiment Details} \label{sec:exp_det}

\subsection{Estimation} \label{ssec:estexpdet}

To implement the density ratio estimation from \citet{choi2022density}, we utilize the codebase provided by the authors\footnote{https://github.com/ermongroup/dre-infinity}. Following the success in their paper, we choose to implement the ``Joint'' model where the model, in addition to time scores, aims to recover data scores. For our interpolation technique, we utilize $p_\lambda(x) = \lambda P_0(x) + \sqrt{1 - \lambda^2} P_1(x)$. For the model architecture, we utilize the Joint architecture proposed in Appendix F of \citet{choi2022density} which leverages multilayer perceptrons with Exponential Linear Unit (ELU) activations as the modules: 

\begin{enumerate}[label=\arabic*.]
  \item \textbf{Joint (Shared)}:
    \begin{align*}
      &\mathrm{Linear}(3,256)
      \;\to\;\mathrm{ELU}
      \;\to\;\mathrm{Linear}(256,512)
      \;\to\;\mathrm{chunk}(2)
    \end{align*}
    \begin{enumerate}[label=(\alph*)]
      \item \textbf{Time Module}:
        \begin{align*}
          &\mathrm{Linear}(256,256)
          \;\to\;\mathrm{ELU}
          \;\to\;\mathrm{Linear}(256,256)
          \;\to\;\mathrm{ELU}
          \;\to\;\mathrm{Linear}(256,1)
        \end{align*}
      \item \textbf{Data Module}:
        \begin{align*}
          &\mathrm{Linear}(256,256)
          \;\to\;\mathrm{ELU}
          \;\to\;\mathrm{Linear}(256,256)
          \;\to\;\mathrm{ELU}
          \;\to\;\mathrm{Linear}(256,2).
        \end{align*}
    \end{enumerate}
\end{enumerate}

We utilize a Cosine Annealing learning rate scheduler with an initial learning rate of 1e-5 and a batch size of 128. We train all models for 30,000 steps with no importance weighting. All other parameters are standard to the original implementation. All models are trained using NVIDIA L40S GPUs. We highlight an ablation study over hyperparameters in Section \ref{sec:ablation}. For the Monte Carlo method, we utilize all available data; we find that this yields more accurate estimation. To solve the infimum, we utilize Brent's algorithm, although we highlight that this is not required following Lemma~\ref{lemma:convex} (any convex optimization method would suffice).

\subsubsection{Time Complexity}

The overall time complexity of our Chernoff Information estimator can be given as:

\begin{equation}
T_\text{total} = T_\text{DRE} + n \cdot  C_\text{int} + n \cdot C_\text{MC} + T_\text{cvx}
\end{equation}

where $n$ is the number of samples used to compute Chernoff Information, $T_\text{DRE}$ is time spent on density ratio estimation (DRE) in Line 1, $C_\text{int}$ is the time required for one integral operations given in Line 2,  $C_{MC}$ is the time required for simple addition/floating point operations required in Line 3, and $T_\text{cvx}$ is the time required to solve the convex optimization in Line 4, which is a constant time operation. Hence, for other than $T_\text{DRE}$, the rest of the operation is done in $O(n)$ and training a neural network for DRE is often the most time-consuming part, which took approximately 30 minutes on a single NVIDIA L40S GPU.

\subsection{Gaussian Estimation} \label{ssec:gaussexpdet}

For all Gaussian experiments, we estimate Chernoff information using the parameters and algorithm specified in \ref{ssec:estexpdet}. We work with 10,000 samples to perform our estimation. First, to compare to closed form Chernoff Information, we create 2D Gaussians $\mathcal{N}(\textbf{0},  \sigma^2\mathbf{I})$ and $\mathcal{N}(\textbf{1},  \sigma^2\mathbf{I})$ and vary the parameter $\sigma^2$. For each $\sigma^2$, we perform 5 estimates to create error bars.

For more complex 5D Gaussian distributions, we compare to the algorithm provided by \citet{nielsen2022revisiting}. We implement the algorithm directly from the paper and use a stopping criteria of 1e-7. We compute Chernoff Information across different levels of noise added to $\mathcal{N}(\textbf{0}, \frac{1}{2}\mathbf{I})$ and $\mathcal{N}(\textbf{1},\textbf{I})$. We repeat each trial 5 times to obtain error bars.

For the final experiment, we repeat the previous example; however instead of adding noise, we scale up the dimension. We again compare to Gaussian estimation algorithm for distributions $\mathcal{N}(\textbf{0}, \frac{1}{2}\mathbf{I})$ and $\mathcal{N}(\textbf{1},\textbf{I})$ and repeat each trial 5 times to obtain error bars.

\subsection{Datasets} \label{ssec:dataexpdet}

\paragraph{Adult.} For the Adult dataset, we select the continuous features, ``age", ``fnlwgt", ``education-num", ``capital-gain", ``capital-loss", ``hours-per-week", for our experiments. We apply a standard scaler, transforming the distributions so that they are centered at 0 with standard deviation of 1. We consider the standard income classification task and use Sex as the sensitive attribute. To obtain Chernoff Information estimates, we train our model as specified in Section \ref{ssec:estexpdet}. We train each estimate 5 times to create error bars. We compute Chernoff Difference from these Chernoff Information values.

For fairness-accuracy curves, we train split logistic regression models using an 80/20 train-test split for each model. We fix the model for the dominating group (Male) and sweep class weights for the other group, dropping non-Pareto optimal points. We find that the False Negative Rate disparity dominates and thus utilize that in our plot.

\paragraph{HSLS.}

Following \citet{jeong2022fairness}, we select a subset of features from the HSLS dataset. We further restrict to continuous features, giving: ``X1MTHID",``X1MTHUTI", ``X1MTHEFF", ``X1FAMINCOME", ``X1SCHOOLBEL". We apply a standard scaler, giving distributions centered at 0 with standard deviation of 1. For our sensitive attribute create groups using the Race attribute, giving Asian/White and URM (Other). We utilize the ``X1TXMSCR" column to create a binary classification task (top 50th percentile or not). For Chernoff Information and fairness-accuracy curves, we replicate the experiments we conduct on the Adult dataset. We find that the False Positive Rate disparity dominates and thus utilize that in our plot.

\paragraph{MNIST}

We replicate the previous experiments using 2D embeddings of MNIST data which were obtained via a trained autoencoder with latent dimension $d_z=2$. 

\begin{table}[h]
\centering
\caption{Convolutional autoencoder architecture used for MNIST experiments.}
\begin{tabular}{lccc}
\toprule
Layer & Kernel / Units & Stride & Output Shape \\
\midrule
Input & -- & -- & $1 \times 28 \times 28$ \\
Conv2D + ReLU & $3\times3$, 32 ch. & 2 & $32 \times 14 \times 14$ \\
Conv2D + ReLU & $3\times3$, 64 ch. & 2 & $64 \times 7 \times 7$ \\
Flatten & -- & -- & $64\cdot7\cdot7$ \\
Linear & $d_z=2$ & -- & $2$ \\
\midrule
Linear & $64\cdot7\cdot7$ & -- & $64 \times 7 \times 7$ \\
ConvTranspose2D + ReLU & $4\times4$, 32 ch. & 2 & $32 \times 14 \times 14$ \\
ConvTranspose2D + Sigmoid & $4\times4$, 1 ch. & 2 & $1 \times 28 \times 28$ \\
\bottomrule
\end{tabular}
\end{table}

We train the autoencoder for 15 epochs, with a batch size of 256, using the Adam optimizer and mean squared error loss function. We choose digit $3$ for $P_0$, digit $4$ for $P_1$, digit $7$ for $Q_0$, and digit $9$ for $Q_1$. 

\paragraph{Dataset Licenses} All datasets used in this work (Adult, HSLS, MNIST) are publicly available and approved for non-commercial research use under their respective data use agreements..

\subsection{Mixture of Gaussian}

For mixture of Gaussian experiments, we design the distributions as follows. We work with 10,000 samples from each group for our experiments. We utilize the Chernoff Information estimator with the same parameters from Section \ref{ssec:estexpdet} and obtain the fairness accuracy curves by replicating the experimental procedure from Section \ref{sec:chernoff_info}.

\paragraph{Mixture 1}

\begin{align}
P_0 &= \tfrac{1}{2}\,\mathcal{N}\!\left(
  \begin{bmatrix}10.5 \\ 10.5\end{bmatrix},
  6.56\,\textbf{I}
\right)
+ \tfrac{1}{2}\,\mathcal{N}\!\left(
  \begin{bmatrix}10.8 \\ 11.2\end{bmatrix},
  6.43\,\textbf{I}
\right), \\[6pt]
P_1 &= \tfrac{1}{2}\,\mathcal{N}\!\left(
  \begin{bmatrix}0 \\ 0\end{bmatrix},
  1.89\,\textbf{I}
\right)
+ \tfrac{1}{2}\,\mathcal{N}\!\left(
  \begin{bmatrix}-1 \\ 2\end{bmatrix},
  \begin{bmatrix}9.43 & 3.3 \\ 3.3 & 9.43\end{bmatrix}
\right), \\[6pt]
Q_0 &= \tfrac{1}{2}\,\mathcal{N}\!\left(
  \begin{bmatrix}0.5 \\ 0.5\end{bmatrix},
  0.55\,\textbf{I}
\right)
+ \tfrac{1}{2}\,\mathcal{N}\!\left(
  \begin{bmatrix}0.5 \\ 0.5\end{bmatrix},
  0.45\,\textbf{I}
\right), \\[6pt]
Q_1 &= \tfrac{1}{2}\,\mathcal{N}\!\left(
  \begin{bmatrix}3.8 \\ 3.8\end{bmatrix},
  \begin{bmatrix}0.2 & 0.1 \\ 0.1 & 0.2\end{bmatrix}
\right)
+ \tfrac{1}{2}\,\mathcal{N}\!\left(
  \begin{bmatrix}2.8 \\ 2.8\end{bmatrix},
  0.55\,\textbf{I}
\right).
\end{align}

\begin{figure}[!ht]
  \centering
  \begin{subfigure}{0.45\columnwidth}
    \centering
    \includegraphics[width=\linewidth]{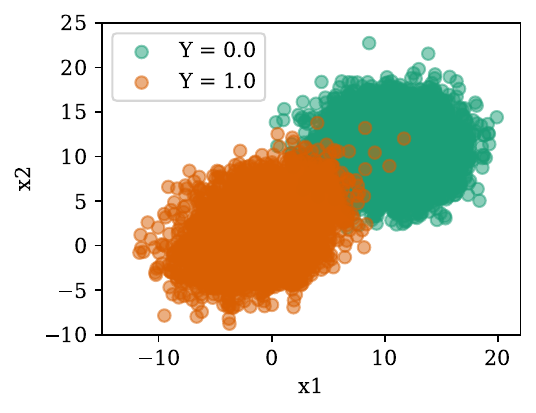}
    \caption{Group $P$}
  \end{subfigure}
  \begin{subfigure}{0.45\columnwidth}
    \centering
    \includegraphics[width=\linewidth]{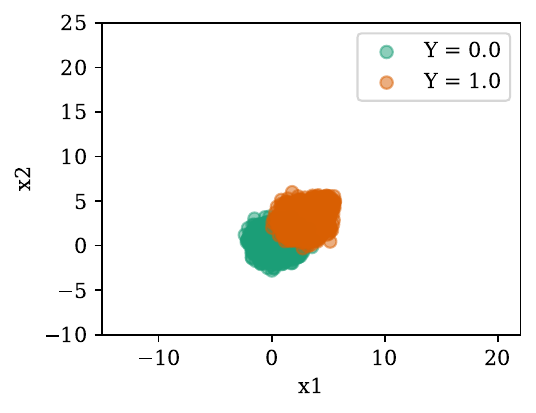}
    \caption{Group $Q$}
  \end{subfigure}
  \caption{{\footnotesize \textbf{(Mixture 1)}}}
\end{figure}

\paragraph{Mixture 2}
\begin{align}
P_0 &= \tfrac{1}{2}\,\mathcal{N}\!\left(
  \begin{bmatrix}0.2 \\ 0.1\end{bmatrix},
  0.2\,\textbf{I}
\right)
+ \tfrac{1}{2}\,\mathcal{N}\!\left(
  \begin{bmatrix}0.5 \\ 0.2\end{bmatrix},
  \begin{bmatrix}0.25 & 0.1 \\ 0.1 & 0.25\end{bmatrix}
\right), \\[6pt]
P_1 &= \tfrac{1}{2}\,\mathcal{N}\!\left(
  \begin{bmatrix}3.1 \\ 2.4\end{bmatrix},
  \begin{bmatrix}0.22 & 0.19 \\ 0.19 & 0.22\end{bmatrix}
\right)
+ \tfrac{1}{2}\,\mathcal{N}\!\left(
  \begin{bmatrix}2.9 \\ 3.2\end{bmatrix},
  0.24\,\textbf{I}
\right), \\[6pt]
Q_0 &= \tfrac{1}{2}\,\mathcal{N}\!\left(
  \begin{bmatrix}-4.2 \\ -5.2\end{bmatrix},
  9\,\textbf{I}
\right)
+ \tfrac{1}{2}\,\mathcal{N}\!\left(
  \begin{bmatrix}-6.2 \\ -3.2\end{bmatrix},
  \begin{bmatrix}9 & 5 \\ 5 & 9\end{bmatrix}
\right), \\[6pt]
Q_1 &= \tfrac{1}{2}\,\mathcal{N}\!\left(
  \begin{bmatrix}1.3 \\ 1.3\end{bmatrix},
  \begin{bmatrix}9 & 4 \\ 4 & 9\end{bmatrix}
\right)
+ \tfrac{1}{2}\,\mathcal{N}\!\left(
  \begin{bmatrix}1.3 \\ 1.3\end{bmatrix},
  8\,\textbf{I}
\right).
\end{align}

\begin{figure}[!ht]
  \centering
  \begin{subfigure}{0.45\columnwidth}
    \centering
    \includegraphics[width=\linewidth]{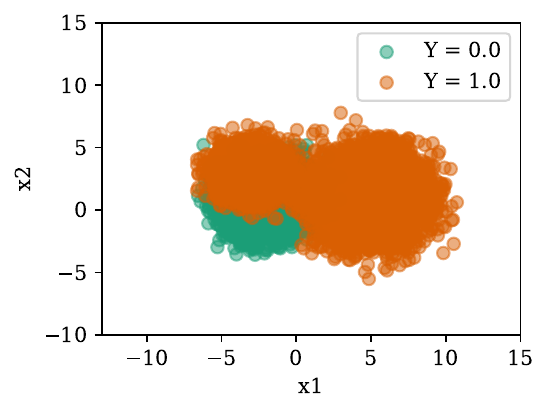}
    \caption{Group $P$}
  \end{subfigure}
  \begin{subfigure}{0.45\columnwidth}
    \centering
    \includegraphics[width=\linewidth]{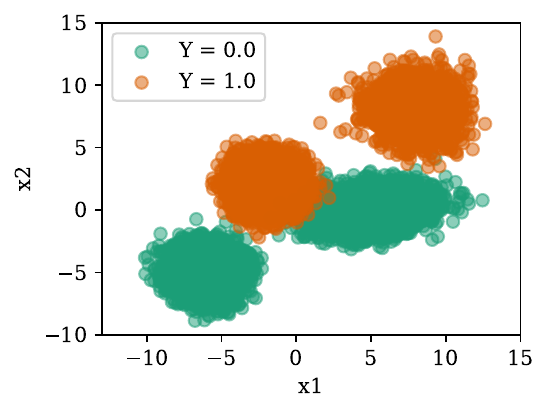}
    \caption{Group $Q$}
  \end{subfigure}
  \caption{{\footnotesize \textbf{(Mixture 2)}}}
\end{figure}

\newpage
\section{HSLS} \label{sec:hsls}

\begin{figure*}[h]
  \centering
  \begin{subfigure}{0.302\linewidth}
    \centering
\includegraphics[width=\linewidth]{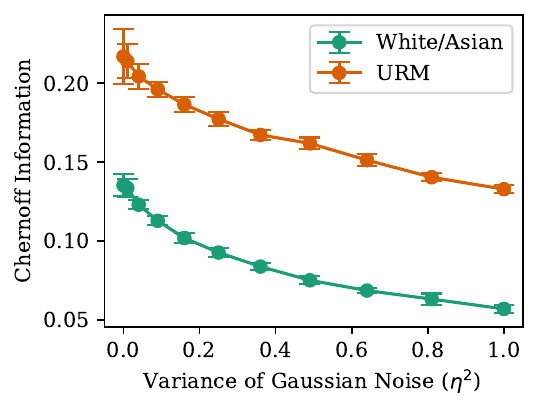}
    \caption{}
    \label{fig:hslsa}
  \end{subfigure}
  \begin{subfigure}{0.302\linewidth}
    \centering    \includegraphics[width=\linewidth]{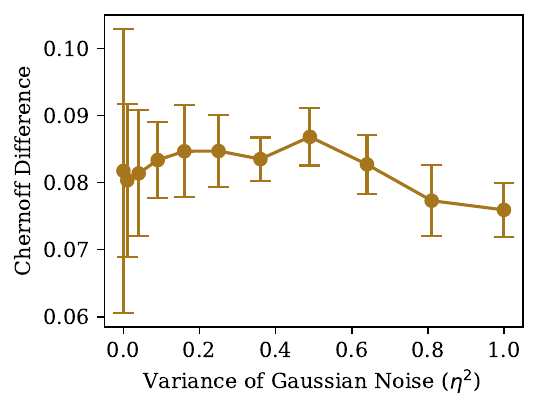}
    \caption{}
    \label{fig:hslsb}
  \end{subfigure}
  \begin{subfigure}{0.382\linewidth}
    \centering
\includegraphics[width=\linewidth]{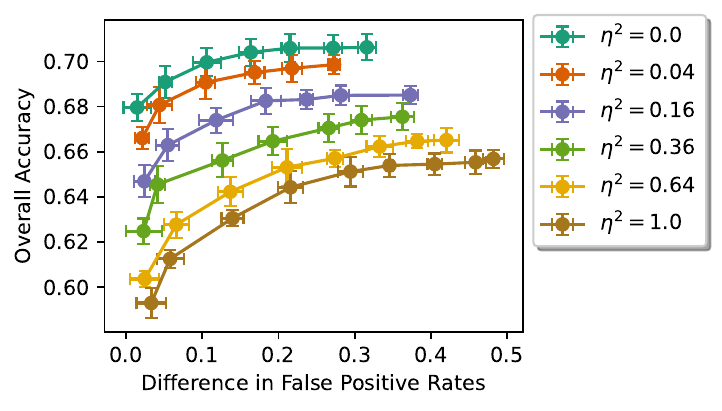}
    \caption{}
    \label{fig:hslsc}
  \end{subfigure}
  \caption{{\footnotesize \textbf{(HSLS Experiments)} (a) Chernoff Information for White/Asian and URM groups decrease as noise ($\eta^2$) increases. (b) Chernoff Difference values are larger with a small increase but still remain flat (Case 1). (c) Following the low Chernoff Difference values, are steeper and grow slightly more steep as noise is added.}
  }
\end{figure*}

\paragraph{HSLS.} We examine the behavior of Chernoff Information on the HSLS dataset \citep{ingels2011hsls}. 
Following \citet{jeong2022fairness}, we select a subset of the features (further filtering to include only continuous features) to perform a binary classification task (top 50th percentile mathematics test score). This gives a 5 dimensional dataset on which we perform our analysis. For the binary groups, we split the data into two groups based on Race, Asian/White and Under Represented Minorities (URM). We follow a similar approach to the Adult dataset to create the Chernoff Difference plots. Additionally, we follow a similar approach to generate fairness-accuracy plots, however, we find that the false positive rate dominates the error and the Chernoff Exponent, using that as our fairness metric.

Again, the Chernoff Information decays as noise is added, however as opposed to the Adult dataset, there is a much larger disparity between the value across groups. This is reflected as we see much larger values of Chernoff Difference in Figure \ref{fig:hslsb} and steeper fairness accuracy curves (Figure \ref{fig:hslsc}). 

Similar to the Adult dataset, we observe the Chernoff Difference remains relatively stable as the variance of noise ($\eta^2$) increases, however, we do note a small increase before the Chernoff Difference begins to slightly decrease (indicative of Case 1). We see that this trend is somewhat reflected, by a very subtle steepening of the fairness accuracy curve (we point the reader to Appendix \ref{sec:suppdat} for the log-fairness accuracy curve where this very subtle trend can be more easily examined). These observations could resemble Case 1, however, we point out that this trend is very subtle and for the most part, the slopes remain flat. Overall, these experiments suggest that in this dataset, privacy and fairness may be slightly less compatible, although the effect is minor.

\newpage
\section{Supplemental Real Data Figures} \label{sec:suppdat}

\begin{figure}[ht]
  \centering
  \begin{subfigure}{0.48\columnwidth}
    \centering
    \includegraphics[width=\linewidth]{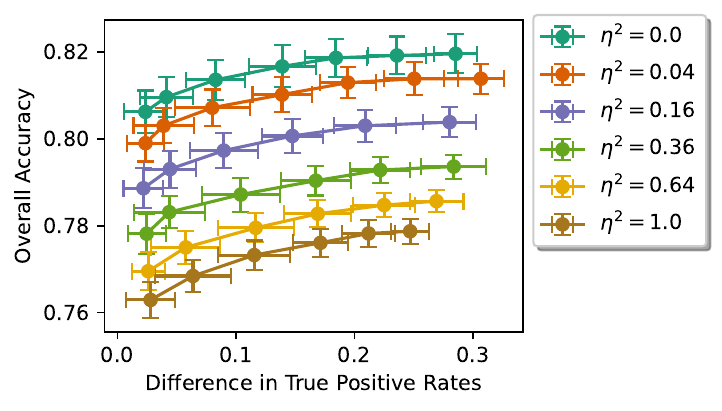}
    \caption{}
    \label{fig:caseBAdult}
  \end{subfigure}
  \hfill
  \begin{subfigure}{0.48\columnwidth}
    \centering
    \includegraphics[width=\linewidth]{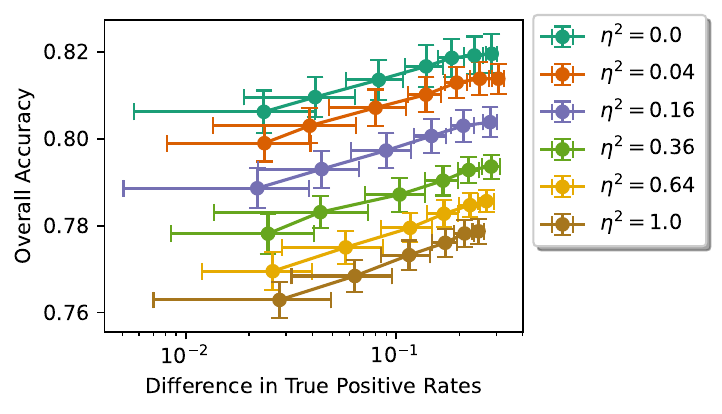}
    \caption{}
    \label{fig:caseBAdultB}
  \end{subfigure}
  \caption{{\footnotesize \textbf{(Adult Dataset)} (a) Fairness-Accuracy Curve. (b) Log Fairness-Accuracy Curve. We find that the curves remain very stable, reflected by the small changes in Chernoff Difference.}}
\end{figure}

\begin{figure}[ht]
  \centering
  \begin{subfigure}{0.48\columnwidth}
    \centering
    \includegraphics[width=\linewidth]{imgs/figure6c.pdf}
    \caption{}
    \label{fig:caseBhsls}
  \end{subfigure}
  \hfill
  \begin{subfigure}{0.48\columnwidth}
    \centering
    \includegraphics[width=\linewidth]{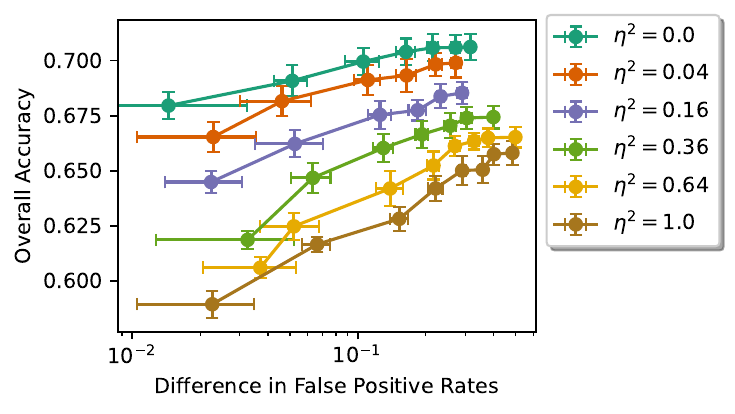}
    \caption{}
    \label{fig:caseBhslsB}
  \end{subfigure}
  \caption{{\footnotesize \textbf{(HSLS Dataset)} (a) Fairness-Accuracy Curve. (b) Log Fairness-Accuracy Curve. We observe a very slight steepening of the fairness accuracy curves, further emphasized by the log-fairness accuracy curves.}}
\end{figure}

\newpage
\section{Ablation Study} \label{sec:ablation}

All experiments are performed on 5D Gaussians $\mathcal{N}(\textbf{0},  \frac{1}{2}\mathbf{I})$ and $\mathcal{N}(\textbf{1},  \mathbf{I})$, with the standard setup mentioned in Section \ref{ssec:estexpdet} unless otherwise stated. All plots are Chernoff Information estimates across steps of the training procedure for the density ratio estimation.

\begin{figure*}[!htp]
    \centering
    \includegraphics[width=0.75\linewidth]{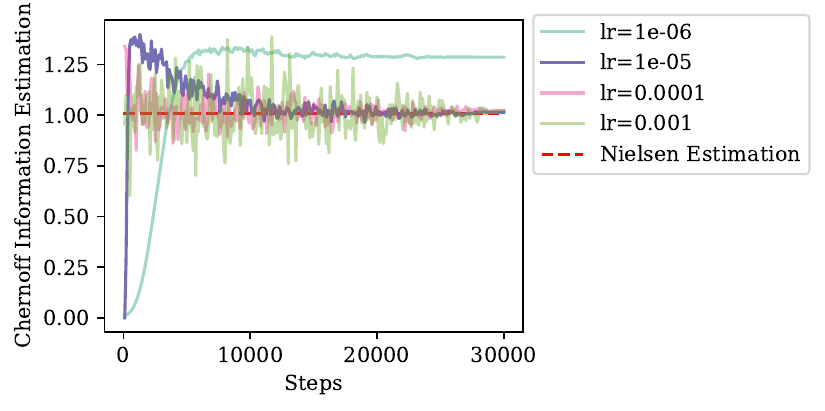}
    \caption{\textbf{(Learning Rate)}: Comparison of Chernoff Information Estimation for different learning rates. We find learning rate is the most sensitive hyperparameter and may converge to an incorrect value if too small or be unstable if too large. We explore a selection of learning rates and find that 1e-5 works well under Gaussian and tabular settings.}
\end{figure*}

\begin{figure*}[t!h]
    \centering
    \includegraphics[width=0.75\linewidth]{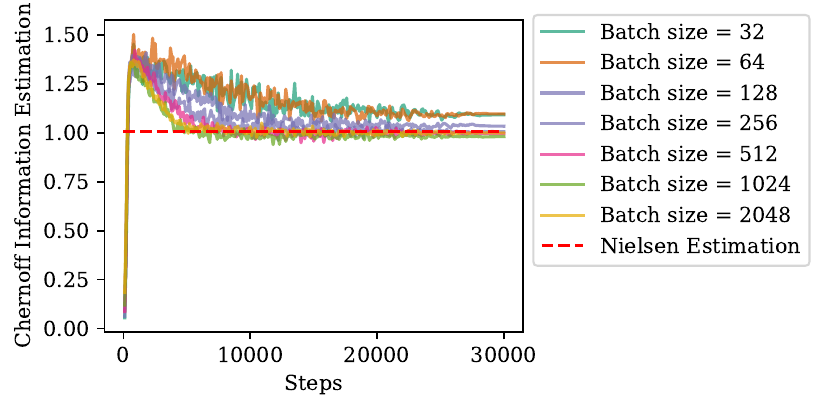}
    \caption{\textbf{(Batch Size)} Comparison of Chernoff Information Estimation for different batch sizes. We find that batch size has a more subtle effect on the estimation outcome. Larger batch sizes tend to be slightly more beneficial for estimation. However, the sample sizes of the tabular datasets we use are smaller than the synthetic Gaussian experiments, causing us to scale the batch size down to 128 for our experiments. These larger batch sizes lead to explosion of the density ratio estimates in tabular settings. We find that this slightly smaller value performs well on both the Gaussian and Tabular settings and prevents the density ratio from exploding in the later case.}
\end{figure*}
\clearpage
\vspace*{0pt}
\begin{figure*}[h!]
    \centering
    \includegraphics[width=0.75\linewidth]{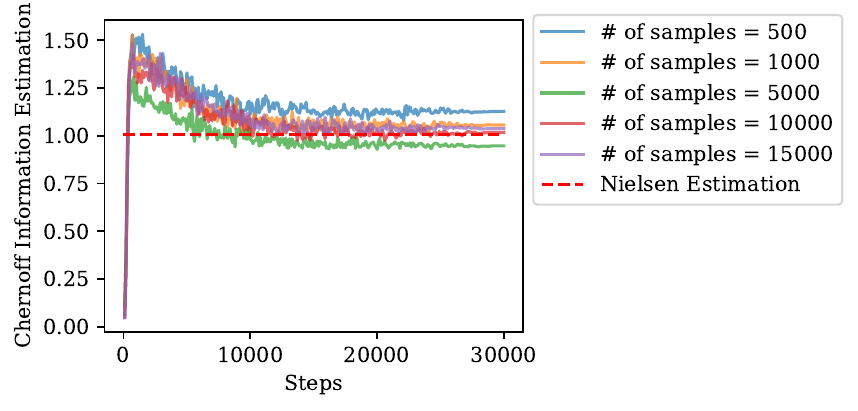}
    \caption{\textbf{(Sample Size)} Comparison of Chernoff Information Estimation for sample sizes. We find that even down to 1000 samples (sample sizes similar to the smallest group in the tabular datasets), the estimation remains fairly accurate. When there are not enough samples, the estimate begins to degrade, as we can begin observing with 500 samples. This sample size and dimension relationship is further highlighted in Figure \ref{fig:gaussc}.}
\end{figure*}

\end{document}